\documentclass[sigconf]{acmart}

\usepackage{amsmath,amsfonts}
\usepackage{amsthm}
\usepackage{pdfpages}
\usepackage{mathtools}
\usepackage{algorithm}
\usepackage{algpseudocode}
\usepackage{graphicx}
\usepackage{textcomp}
\usepackage{xcolor}
\usepackage{url}  
\usepackage{booktabs}  
\usepackage{nicefrac}       
\usepackage{microtype}      
\usepackage[caption=false,font=footnotesize]{subfig}
\usepackage{enumitem}
\usepackage[table]{xcolor}
\usepackage{multirow}
\usepackage{booktabs}

\theoremstyle{plain}
\newtheorem{theorem}{Theorem}[section]

\theoremstyle{definition}

\theoremstyle{remark}

\DeclareMathOperator{\IQ}{IQ}
\DeclareMathOperator{\I}{I}

\DeclareMathOperator{\NMI}{NMI}
\DeclareMathOperator{\D}{d}

\DeclareMathOperator{\len}{len}

\AtBeginDocument{%
  }

\copyrightyear{2026}
\acmYear{2026}
\setcopyright{cc}
\setcctype{by}
\acmConference[KDD '26]{Proceedings of the 32nd ACM SIGKDD Conference on Knowledge Discovery and Data Mining V.2}{August 09--13, 2026}{Jeju Island, Republic of Korea}
\acmBooktitle{Proceedings of the 32nd ACM SIGKDD Conference on Knowledge Discovery and Data Mining V.2 (KDD '26), August 09--13, 2026, Jeju Island, Republic of Korea}

\begin{document}

\title{Estimating Mutual Information between Time Series and Temporal Event Sequences Across Diverse Analysis Tasks}

\author{Haoji Hu}
\authornote{Both authors contributed equally to this research. This paper has been accpeted to KDD 2026.}
\email{huxxx899@umn.edu}
\affiliation{%
  \institution{University of Minnesota - Twin Cities}
  \city{Minneapolis}
  \state{MN}
  \country{USA}
}

\author{Huaqing Mao}
\authornotemark[1]
\email{maohqphy@hotmail.com}
\affiliation{%
  \institution{University of Minnesota - Twin Cities}
  \city{Minneapolis}
  \state{MN}
  \country{USA}
}

\author{Yijun Lin}
\email{lin00786@umn.edu}
\affiliation{%
  \institution{University of Minnesota - Twin Cities}
  \city{Minneapolis}
  \state{MN}
  \country{USA}
}

\author{Xiaowei Jia}
\email{xiaowei@pitt.edu}
\affiliation{%
  \institution{University of Pittsburgh}
  \city{Pittsburgh}
  \state{PA}
  \country{USA}
}

\author{Jinwei Zhou}
\email{zhou1909@umn.edu}
\affiliation{%
  \institution{University of Minnesota - Twin Cities}
  \city{Minneapolis}
  \state{MN}
  \country{USA}
}

\author{Minoh Jeong}
\email{minohj@inha.ac.kr}
\affiliation{%
  \institution{Inha University}
  \city{Incheon}
  \state{Incheon}
  \country{Republic of Korea}
}

\author{Yao-Yi Chiang}
\email{yaoyi@umn.edu}
\affiliation{%
  \institution{University of Minnesota - Twin Cities}
  \city{Minneapolis}
  \state{MN}
  \country{USA}
}

\renewcommand{\shortauthors}{Haoji Hu et al.}

\begin{abstract}
Pairwise dependence measures such as correlation and causality are fundamental to temporal data mining, yet there is still no principled and robust way to quantify dependence between heterogeneous data types, especially between continuous time series and discrete temporal event sequences. Existing approaches rely on ad hoc transformations or mutual-information estimators that are highly sensitive to quantization, repeated values, and event redundancy, leading to biased or unstable results in practice. We propose a nonparametric mutual information estimator that directly measures the dependence between time series and event sequences without data transformation, learning, or ad hoc discretization. Our method models the continuous–discrete duality of real-world time series to handle quantization and repeated-value artifacts and introduces a latent event clustering strategy to mitigate bias from event co-occurrence and redundancy. Together, these yield a robust and unified framework that bridges discrete and continuous mutual information. We evaluate the proposed estimator on four representative tasks: discrete–continuous time-delayed mutual information for causality analysis, global and local temporal repetition discovery, discrete covariate selection for time series forecasting, and continuous feature selection for classification. Experiments on synthetic and real-world datasets show consistent improvements over existing methods in accuracy, robustness, and interpretability, positioning our approach as a general-purpose dependence operator for heterogeneous temporal data, similar to Pearson correlation for homogeneous time series. Code: \url{https://github.com/HaojiHu/Multimodal-Temporal-Data-Quantification}
\end{abstract}



%


\maketitle

\section{Introduction}
Time series and temporal event sequences are two of the most common and fundamental forms of temporal data in real-world applications. Time series represent continuous-valued measurements evolving over time, such as temperature, traffic volume, or magnetic field intensity, whereas temporal event sequences consist of discrete events characterized by timestamps and symbolic event types, such as system alerts, transactions, or scientific phenomena. These two data types naturally coexist in many domains, where continuous physical or socioeconomic processes are influenced by, or give rise to, discrete events, and vice versa.

Quantifying pairwise relationships between \textbf{time series} and \textbf{temporal event sequences} is a core problem in temporal data mining and supports a wide range of downstream tasks, including correlation and causality analysis, temporal pattern discovery, and feature or covariate selection. For example, in solar physics, estimating the time-delayed mutual information (TDMI) between discrete solar event sequences and continuous geomagnetic field measurements is crucial to understanding how and when solar activity influences Earth’s magnetic field~\cite{chitta2023picoflare}. Similar heterogeneous temporal relationships appear in transportation systems (traffic volume versus calendar events) and sales forecasting (sales time series versus promotions), highlighting their broad practical importance.

Most existing dependence measures focus on homogeneous temporal data. Pearson correlation and its variants quantify relationships between two time series (Figure~\ref{fig:example}(a)), while point-wise mutual information measures dependencies between two temporal event sequences (Figure~\ref{fig:example}(b)). These methods are effective within a single data type but do not directly support cross-type temporal analysis, such as measuring the interaction between a continuous time series and a discrete temporal event sequence (Figure~\ref{fig:example}(c)).
A common workaround is to transform one temporal data type into another and then apply standard homogeneous measures. Typical approaches include discretizing time series into symbolic sequences via binning~\cite{liu2002discretization}, or converting event sequences into numerical series by assigning integer identifiers (IDs) to symbols~\cite{ross2014mutual}. Such transformations introduce ad hoc design choices (e.g., bin width or symbol ordering) and impose non-existent artificial structures, leading to biased estimates and reduced interpretability. Some prior work has explored event–time series relationships using hypothesis testing that analyzes signal behavior around event occurrences~\cite{luo2014correlating}, but does not explicitly quantify dependence. More importantly, existing approaches fail to account for real-world characteristics of temporal data, such as sensor quantization effects, repeated values in time series, and correlations among events in temporal event sequences.

\begin{figure*}
\centering
\includegraphics[width=0.87\textwidth, height=0.13\textheight]{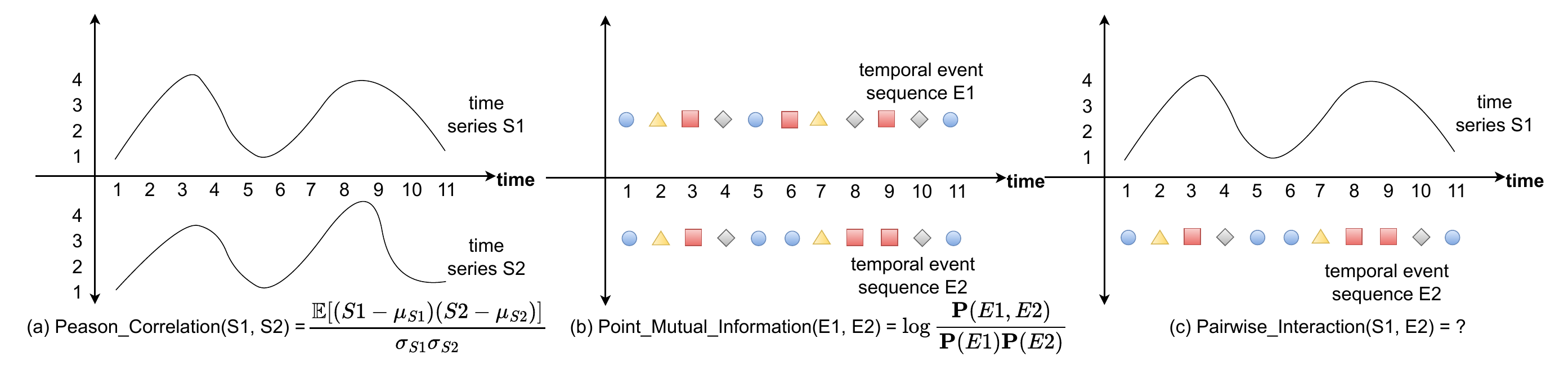}
\captionsetup{font=small}
  \caption{Dependence analysis across temporal data types. (a) Pearson correlation between two time series, where each time step is a continuous value. (b) Point-wise MI between two temporal event sequences, where each time step is a discrete event (shape + color), and marginal and joint probabilities are estimated from event frequencies. (c) Unknown cross-type interaction between a time series and a temporal event sequence.}
\label{fig:example} 
\end{figure*}

Mutual information (MI) is a model-free, principled measure of statistical dependence that applies to both discrete and continuous variables, making it a natural choice for cross-type temporal analysis. However, existing MI estimators do not adequately handle heterogeneous temporal data in practice. They face two fundamental challenges: (1) real-world time series often violate strict continuity assumptions of continuous variables due to finite precision and repeated values; for example, temperature recorded at integer or low-decimal resolution usually takes identical values across many timestamps. 
(2) Event types are typically treated as independent symbols, ignoring redundancy in their induced distributions over continuous values. These mismatches between theoretical assumptions and real data lead to unstable and biased estimates.

In this paper, we propose a nonparametric MI estimator that directly quantifies dependence between a time series and a temporal event sequence, without data transformation, discretization heuristics, or learning-based modeling. Our method introduces two key ideas:
(1) a \emph{continuous–discrete duality representation} that explicitly models the coexistence of continuous variation and repeated values caused by finite measurement precision;
and (2) an optional \emph{latent event clustering} strategy that aggregates redundant or highly correlated event types into representative latent events.
Together, our proposed estimator unifies classical discrete and continuous MI and serves as a general-purpose dependence operator for heterogeneous temporal data.
We demonstrate the generality and effectiveness of the proposed estimator through four representative tasks involving time series and temporal event sequences.

\noindent \textbf{(1) Time-delayed mutual information estimation}, extending the quantification of  lagged dependencies from homogeneous time series to  time series–event sequence pair;

\noindent\textbf{(2) Temporal repetition analysis}, identifying repeated temporal structures in time series, including both global seasonality and localized, event-conditional repetition patterns;

\noindent\textbf{(3) Covariate analysis for time series forecasting}, ranking discrete covariates by their relevance to a continuous target time series to support reliable and interpretable forecasting;
and 

\noindent\textbf{(4) Extension to mixed-type tabular data analysis}, extending beyond temporal settings to enable feature selection by measuring the dependence between continuous features and discrete labels.

Extensive experiments on synthetic and real-world datasets show that the proposed measurement consistently outperforms existing methods and offers interpretable results, establishing it as a foundational building block for heterogeneous temporal data mining. In summary, our main contributions are:

\noindent\textbf{(1)} A nonparametric MI estimator that directly quantifies dependence between a time series and a temporal event sequence while accounting for quantization and repeated values.

\noindent\textbf{(2)} A clustering strategy that reduces redundancy from event co-occurrence in event sequences.

\noindent\textbf{(3)} An empirical evaluation demonstrating effectiveness across core data mining tasks, including cross-modal TDMI estimation, temporal repetition pattern discovery, discrete covariate selection for forecasting, and continuous feature selection for classification.

\section{Related Work}

\subsection{Mutual Information Estimation}
Estimating mutual information (MI) for continuous variables is a long-standing challenge~\cite{walters2009estimation}. Existing parametric methods rely on strong assumptions about the underlying data distribution (e.g., Gaussianity), which are often violated in real-world temporal data. In contrast, nonparametric methods are more flexible and make fewer assumptions, making them the dominant choice in practice.

Nonparametric MI estimators such as KSG~\cite{kraskov2004estimating} and its extensions~\cite{nagel2024accurate} are built upon the Kozachenko–Leonenko entropy estimator~\cite{kozachenko1987sample}. These methods assume strictly continuous variables, under which the probability of observing identical samples is zero (i.e., the nearest-neighbor distance is always positive).
In real-world time series, however, values are observed with finite precision due to sensor limitations, rounding, or quantization, easily leading to repeated values. When duplicates occur, nearest-neighbor distances collapse to zero, producing degenerate neighborhoods and unstable entropy or MI estimates. This issue is often ignored, heuristically mitigated, or assumed to vanish in high-dimensional data, but it becomes critical in low-dimensional temporal settings (e.g., univariate time series). We address this mismatch between theoretical assumptions and practical data by modeling a time series as a mixture of continuous and discrete components. This formulation treats repeated values as discrete probability mass rather than anomalous cases, thereby resolving a core limitation of existing nonparametric MI estimators without resorting to ad hoc fixes.
Recent progress like~\cite{belghazi2018mutual}\cite{chen2024addressing} learn an estimator. But, these methods require model training and fail to directly estimate the MI for just one pair data.

\subsection{Temporal Data Analysis}
Temporal dependence analysis has been extensively studied for homogeneous data types. Classical approaches such as canonical correlation analysis~\cite{akaike1976canonical}, co-occurrence analysis~\cite{church1990word}, and sequence similarity measures~\cite{magallanes2021sequen} focus on quantifying dependence within a single temporal data type.
In contrast, cross-type temporal analysis between continuous time series and discrete-event sequences has received far less attention. Yet, such heterogeneous temporal relationships are ubiquitous in real-world applications. In this paper, we focus on three representative tasks that highlight the importance of this problem. We review related work for each task below.

\noindent
\textbf{Time-Delayed Mutual Information Estimation.}
Existing TDMI methods assume both variables are continuous time series and are therefore not applicable when one variable is a temporal event sequence. While discrete–continuous MI estimators~\cite{ross2014mutual,gao2017estimating} can in principle be applied, they rely on neighborhood-based density estimation, inherit strict continuity assumptions, and ignore quantization artifacts and event redundancy. Our approach directly models time series–event sequence interactions, enabling robust and interpretable cross-type TDMI estimation.

\noindent
\textbf{Temporal Seasonality and Repetition Analysis.}
Seasonality detection has traditionally relied on analyzing a time series as a continuous signal. Many techniques have been developed for this purpose, including autocorrelation function (ACF)~\cite{box2015time}, seasonal decomposition method STL~\cite{cleveland1990stl}, and Fourier transform (FT)~\cite{brillinger2001time}. 

ACF detects periodicity by identifying peaks in specific lags. Fourier methods identify dominant frequencies associated with repeating patterns. Despite their effectiveness, these techniques are sensitive to outliers and missing values, as each time step is modeled deterministically. Such irregularities can significantly distort the resulting ACF or FT.

Our method reframes seasonality as a dependence problem between a time series and a temporal event sequence (e.g., calendar structure), enabling a probabilistic treatment that is robust and unifies global and localized repetition patterns.

\noindent
\textbf{Covariate Analysis and Feature Selection.}
Feature and covariate selection methods are two parallel analysis tasks that share common problem structures, especially for continuous feature selection and discrete covariate selection. Both model the interaction between continuous values and discrete symbols, i.e., selecting continuous features for discrete-label prediction and discrete covariates for time-series forecasting. Although model-based methods have been proposed~\cite{guyon2003feature}\cite{belmonte2014hierarchical}, due to model bias, model-agnostic methods are more widely used. Information-theory-based methods~\cite{dougherty1995supervised,liu2002discretization} are representative model-agnostic methods. However, existing methods often rely on discretization when continuous variables are involved. A common strategy is binning~\cite{dougherty1995supervised,liu2002discretization}, which partitions the continuous value range into $k$ intervals and assigns each interval a discrete symbol. Although simple and widely used, binning introduces a strong inductive bias through the choice of bin boundaries and widths. Moreover, discretization assumes a one-to-one mapping between continuous values and discrete symbols. The same continuous value may be associated with multiple discrete events, and forcing a hard partition can distort this relationship and reduce the reliability of feature relevance estimates. 

These limitations motivate the need for a measure that can directly quantify relationships between continuous values and discrete events. Our proposed estimator addresses this gap and provides a principled alternative for covariate analysis and feature selection in heterogeneous data.

\section{Information Theory based Measurement}

\noindent
\textbf{Problem definition}: Given a time series of continuous values $\mathbf{S} = (s_1, s_2, ..., s_m)$, where $m$ represents the number of time steps, and a temporal event sequence of discrete symbols $\mathbf{E} = (e_1, e_2, ..., e_n)$, where $n$ is the number of observed events, our goal is to quantify their interaction, denoted as $\IQ (\mathbf{S}, \mathbf{E})$, which measures the degree of correlation between temporally aligned elements of $\mathbf{S}$ and $\mathbf{E}$. The time indices encode only relative temporal order; two modalities need not share identical timestamps and may be sampled asynchronously across data sources.

Our proposed measurement builds on mutual information (MI), a principled measure of statistical dependence between two random variables. We first present an MI formulation between a discrete and a continuous variable by unifying Shannon MI for discrete variables with differential MI for continuous variables (Section~\ref{sec: 31}). 
For continuous-valued time series $\mathbf{S}$, we introduce a continuous–discrete duality representation that models the empirical distribution as a mixture of a continuous component and a precision-induced discrete component, and we extend this formulation to estimate MI between a mixture variable and a discrete variable (Section~\ref{sec: 32}).
For a discrete event sequence $\mathbf{E}$, we introduce an optional clustering strategy that maps event types to a latent event sequence with reduced cardinality, allowing correlated or redundant event types to be modeled jointly (Section~\ref{sec: 33}).

\subsection{MI of Discrete-Continuous Variables}\label{sec: 31}

In Shannon mutual information, we consider two discrete random variables $X$ and $Y$, where $x$ and $y$ denote their values, and $\mathcal{X}$ and $\mathcal{Y}$ denote their domains, i.e.,  $x\in \mathcal{X}$ and $y\in \mathcal{Y}$. If both $\mathcal{X}$ and $\mathcal{Y}$ are subsets of  a countable set of categories $\{\mathcal{C}_1, \mathcal{C}_2, ..., \mathcal{C}_n\}$ (i.e., n is the cardinality size of the countable set and we can get the union set of both sets as the countable set), the MI can be defined based on as,
{\small
\begin{align}
\I(X, Y) = \sum_{y\in \mathcal{Y}}\sum_{x\in \mathcal{X}}P_{(X,Y)}(x,y)\log \left(\frac{P(x, y)}{P(x)P(y)}\right)
\label{eq:mi-dis}
\end{align}
}

In differential mutual information, for two continuous random variables $X$ and $Y$, where the range of $x$ and $y$ is $\mathbb{R}$, the MI is, 
{\small
\begin{align}
\I(X, Y) \!=\! \int_{\mathcal{Y}}\int_{\mathcal{X}}P(x,y)\log \left(\frac{P(x, y)}{P(x)P(y)}\right) \D \!x\D \!y
\label{eq:mi-con}
\end{align}
}

A natural extension of Eqs.~\ref{eq:mi-dis} and~\ref{eq:mi-con} defines the MI between a discrete variable $X$ and a continuous variable $Y$. Specifically, we assume that $X$ takes values from a countable set of categories
$\mathcal{X}={\mathcal{C}_1, \mathcal{C}_2, \ldots, \mathcal{C}_n}$, and $Y$ takes values in the continuous domain $\mathcal{Y}=\mathbb{R}$. The MI between them is,
{\small
\begin{align}
\I(X, Y) = \sum_{x \in \mathcal{X}}\int_{y\in \mathcal{Y}}P(x,y)\log \left(\frac{P(x, y)}{P(x)P(y)}\right) \D y
\label{eq:mi-mix}
\end{align}
}

Equivalently, the MI can be expressed in terms of the entropy as:
{\small
$
\I(X,Y) = H(X) + H(Y) - H(X,Y), 
$
}
We can rewrite Eq.~\ref{eq:mi-mix} as 
{\small
\begin{align}
\I(X, Y) &\!=\! -\!\sum_{x\in \mathcal{X}}\! P(x)\log P(x) \!-\!\! \int_{\mathcal{Y}}\!\!P(y)\log P(y)\!\D\! y \notag\\
&+ \sum_{x\in \mathcal{X}} \int_{\mathcal{Y}}\!P(x,y)\log P(x,y) \!\D\! y
\label{eq:mi}
\end{align}
}

If we fix a specific value of $x$ and only consider $Y$ as a variable, the joint entropy in Eq.~\ref{eq:mi} can be written as
{\small
\begin{align}
\int_{\mathcal{Y}} P(x,y)\log P(x,y)\,\mathrm{d}y
&= p(x)\log p(x) \notag\\
&\quad + p(x)\int_{\mathcal{Y}}
P(y\mid x)\log P(y\mid x)\,\mathrm{d}y
\label{eq:integrate}
\end{align}
}

where the marginal distributions are $P(x) = p(x)$ and 
$P(y) = \sum_{x\in \mathcal{X}} p(x)P(y|x)$. 
The joint probability distribution becomes,
{\small
\begin{align*}
P(x,y) &= p(x) P(y | x), \sum_{x\in\mathcal{X}} p(x) = 1,
\int_{\mathcal{Y}} P(y|x)\,\mathrm{d}y = 1
\end{align*}
}

Based on Eqs.~\ref{eq:mi} and ~\ref{eq:integrate}, the {\small$\I(X, Y)$} between $X$ and  $Y$ is, 
{\small
\begin{align}
\I = \sum_{x\in \mathcal{X}} p(x)\int P(y|x)\log p(y|x) \D\! y -\int P(y)\log P(y)\D \!y 
\label{eq:mi-example}
\end{align}
}

The probability mass function of the discrete variable $X$, denoted by $p(x)$, can be directly estimated from data using empirical frequencies of the corresponding discrete value $x$. In contrast, estimating the distribution of a continuous variable is more challenging. A common approach is to assume a Gaussian distribution, which is parametric and relies on the strong assumption that the data follow a unimodal normal distribution. This assumption is often unrealistic for temporal data, which frequently exhibit nonstationarity and complex dynamics. To avoid such restrictive assumptions, we adopt a classical nonparametric and asymptotically unbiased entropy estimator introduced by~\cite{kozachenko1987sample}, which directly estimates the entropy of a continuous variable from samples. Specifically, given $N$ samples ${x_i}_{i=1}^{N}$, the entropy is estimated as:
{\small
$
H = ln(\overline{\rho}) + ln(c) + ln(\gamma) + ln(N-1),
\label{eq:ent-est}
$
} where {\small$\displaystyle \overline{\rho} = \{\prod^N_{i=1}\rho_i\}^{1/N}$} (i.e., {\small$\rho_i=min(\{|x_i-x_j|, j\neq i\})$}), {\small$c=\pi^{\frac{1}{2}}/\Gamma(\frac{3}{2})$}, and $ln(\gamma)$ is the Euler constant ($\approx 0.5772$). 


MI takes values in the range $[0, \infty)$, which makes its magnitude difficult to interpret without knowing max pivot. To obtain a normalized measure analogous to  Pearson correlation, which lie in $[0,1]$, it is common to normalize MI. Among the various normalization schemes summarized in~\cite{nagel2024accurate}, we adopt ~\cite{lange2006generalized},
{\small
\begin{align}
\NMI(X, Y) &= \sqrt{1-\exp[-2\I(X,Y)/(\D X + \D Y)]}
\label{eq:nomralization}
\end{align}
}

As this transformation is for continuous variables and the unit $\D X$ is not defined for discrete symbols, we propose an approximation $\D X$=1, as symbols are usually represented by integer IDs. 

Although the entropy of a continuous variable can be estimated using nonparametric entropy estimators, differential entropy is not guaranteed to be nonnegative. By definition, $H(Y)=-\int_{\mathcal{Y}}P_Y(y)\cdot$ $\log P_{Y}(y)\D y$, $\int_{\mathcal{Y}}P_Y(y)\D y=1$, and $\log P_{Y}(y)\in (-\inf, \inf)$. 

Since our MI estimation is derived from entropy terms of continuous variables (see Eq.~\ref{eq:mi-example}), negative differential entropy estimates may propagate and yield negative MI. This violates the fundamental non-negativity property of MI and indicates a failure of the estimation. We provide a formal proof of this property in the Appendix.


This issue can be formally addressed by introducing an invariant measure density $m(y)$ to normalize the logarithmic term and convert differential entropy into a relative entropy form,
{\small
$
H(Y) = \int P(y)\log \frac{P(y)}{m(y)}\, \mathrm{d}y,
$}
as proposed in~\cite{jaynes1968prior}. Fortunately, incorporating $m(y)$ does not affect the value of MI, since the relevant terms cancel out when computing MI (see Eq.~\ref{eq:mi})~\cite{nagel2024accurate}. Therefore, MI remains unchanged regardless of whether the invariant measure is explicitly introduced, and we omit $m(y)$ in our implementation for simplicity.

Another issue that is largely overlooked in prior analyses~\cite{nagel2024accurate} is the possibility of obtaining negative estimates due to finite-sample variability. When the true MI is close to zero, small estimation errors in the entropy terms can cause the estimated values to slightly undershoot their true magnitudes, resulting in negative entropy contributions and, consequently, negative MI. Although such outcomes are theoretically impossible, they can arise in practice due to estimation noise. To guarantee the non-negativity  of MI, we apply a simple truncation: $I \leftarrow \max(I, 0)$. This operation leaves all valid positive estimates unchanged while removing spurious negative values introduced by estimation noise. It provides a numerically stable, theoretically consistent safeguard to improve the robustness without affecting meaningful dependence measurements.

\subsection{Continuous-Discrete Duality}\label{sec: 32}

Modeling a time series purely as a continuous variable with a density function (e.g., $P_Y$) overlooks the fact that real-world measurements often exhibit ordinal characteristics. Ordinal data are categorical in nature but preserve a meaningful ordering among possible values. For example, temperature observations are bounded by physical limits and are recorded with finite precision due to instrument constraints, such as rounding to one decimal place. As a result, the observed values belong to a countable and ordered set rather than a truly continuous domain, a property that cannot be faithfully captured by a purely continuous density model. Such discretization with preserved value ordering makes the observed data analogous to ordinal data. However, temperature itself is a continuous macroscopic variable defined over a thermodynamic system~\cite{landau2013statistical}. The discreteness arises not from the underlying physical process, but from measurement and recording constraints. Consequently, an observed temperature time series is jointly shaped by the continuous nature of the physical phenomenon and the discrete bias introduced by finite-precision instruments. More generally, real-world time series may exhibit purely continuous behavior, purely ordinal behavior, or a combination of both. This motivates a \emph{continuous--discrete duality} in the representation of time series values. That is, a time series variable should simultaneously capture continuous variation and discrete probability mass, and flexibly adjust its relative contributions based on the observed samples.

Formally, we redefine the density function of continuous variable $Y$ as a combination of a continuous distribution $f(y)$ and a discrete distribution $g(y)$, as {\small
$
p(y) = a~ f(y) + b~ g(y)
$
}, where {\small$\int f(y) dy = 1$}, {\small $\sum_{y} g(y) = 1$}, {\small $f(y) \cdot g(y) = 0$}, and {\small $a+b=1$}.

This formulation is interpreted as: each observed timestep value is generated from either the continuous distribution $f(y)$ with probability $a$ or the discrete distribution $g(y)$ with probability $b$.

Given this new formulation of $p(y)$, the entropy of $Y$ can be rewritten with the law of addition for the information entropy applied to both distributions $f(y)$ and $g(y)$ as follows: 
{\small
\begin{align}
H(p) &= \int a~ f(y) ln(a~ f(y)) dy + \sum_y b~ g(y) ln(b~ g(y)) \notag \\
  &= a~ln(a) + a~ H(f(y)) + b~ ln(b) + b~ H(g(y))
\label{eq:Y}
\end{align}
}

For the joint entropy {\small$H(X,Y)$} between the discrete variable $X$ and the redefined continuous–discrete variable $Y$, we decompose it into a weighted combination of two components, denoted by {\small$A(X,Y)$} and {\small$B(X,Y)$}. Here, {\small$A(X,Y)$} represents the joint entropy contributed by the samples in which $Y$ is generated from its continuous component and their aligned values in $X$, while {\small$B(X,Y)$} represents the joint entropy contributed by the samples in which $Y$ is generated from its discrete component and their aligned values in $X$. By the additivity property of entropy over mixture distributions, the joint entropy {\small$H(X,Y)$} can therefore be expressed as a linear combination of these two terms:
{\small
\begin{align}
H(X, Y) = a~ ln(a) + a~ H(A) + b~ ln(b) + b~ H(B)
\label{eq:XY}
\end{align}  
}

With Eqs.~\ref{eq:mi},~\ref{eq:Y}, and~\ref{eq:XY}, the mutual information is: 
{\small
\begin{align}
I(X, Y) = a~ \I_A + b~ \I_B + a~ ln(a) + b~ ln(b)
\label{eq:mi_cont}
\end{align}
}
where {\small$\I_A = \I(X_A, Y_A)$} can be calculated by Eq.~\ref{eq:mi-example} and {\small$\I_B=\I(X_B, Y_B)$} is calculated by Eq.~\ref{eq:mi-dis}.

The remaining question is how to partition the time series $Y$ into its continuous and discrete components. We propose to do so based on the presence of repeated values. Specifically, if a value appears only once in the entire series, we treat the corresponding time step as being generated from the continuous component $f(y)$ and include it in $Y_A$ for computing the MI {\small$I(X_A, Y_A)$}. The event aligned with this time step is accordingly assigned to $X_A$. In contrast, if a value appears more than once, all time steps with this value are treated as being generated from the discrete component $g(y)$ and included in $Y_B$, and their aligned events are included in $X_B$ for the calculation {\small$I(X_B, Y_B)$}. This partitioning is theoretically justified by the assumptions underlying classical nonparametric entropy estimators such as~\cite{kozachenko1987sample}, which are developed for strictly continuous variables, where the probability of observing identical samples is zero. In such a setting, any repeated value violates the continuity assumption, leading to degenerate neighborhoods.

\subsection{Latent Discrete Variable}\label{sec: 33}

In practice, temporal event sequences often contain event types that are highly correlated or redundant. In such cases, distinct event symbols may correspond to nearly identical conditional distributions over the aligned time-series values  $Y$, leading to unnecessary fragmentation and unstable dependence estimates. To address this issue, we introduce an optional event-clustering strategy that aggregates similar event types into latent events. The key idea is to replace redundant symbolic distinctions with a more compact representation that better reflects the data generation process.

This procedure consists of three steps: (1) for each event type, collecting all temporally aligned values from the time series $Y$ to form its empirical conditional distribution; (2) performing hierarchical clustering over these distributions using the Wasserstein distance;
and (3) replacing the original event type with its corresponding cluster label, yielding a latent event sequence. 
This transformation reduces redundancy while preserving the discrete nature and temporal structure. Importantly, the clustering does not alter the data modality and is applied only when strong correlations or overlaps among event types are present.

\section{Experiments}

Evaluating mutual information (MI) estimators is inherently challenging because ground-truth  are not available for real-world data. As a result, prior work has primarily relied on synthetic datasets for controlled evaluation~\cite{ross2014mutual, gao2017estimating}. Building on this practice, we first validate our method using carefully designed synthetic data with known ground truth. Beyond synthetic evaluation, we further assess the effectiveness of our measurement through a set of real-world data and analytic tasks in which MI serves as a core component. Improvements in task performance provide complementary and practical evidence of the utility and robustness of the proposed estimator in realistic settings.

\subsection{Time-Delayed Mutual Information}

This task evaluates the estimation of time-delayed mutual information (TDMI) between a time series $S$ and a temporal event sequence $E$. Given a time lag $\tau$, TDMI measures the dependence between the event type occurring at time $t$ in $E$ and the value of the time series at time $t+\tau$. A key advantage of the proposed measurement is that it natively supports heterogeneous variables, allowing the time delay to be incorporated directly into the MI formulation without transforming either modality. Specifically, TDMI is computed as
{\small
\begin{displaymath}
    \text{TDMI}(E(t), S(t+\tau)) = I(E(t), S(t+\tau)).
\end{displaymath}
}

\noindent \textbf{Dataset.} 
We construct a synthetic dataset to obtain ground-truth MI values. A temporal event sequence is generated by random sampling from three events (IDs are 1, 2, and 3) with equal probability. Each event  aligns with a distinct Gaussian distribution. The corresponding distributions are $\mathcal{N}(-0.7, 0.02^2)$, $\mathcal{N}(0, 0.02^2)$,  $\mathcal{N}(0.7, 0.02^2)$. When we sample an event, a corresponding continuous value is sampled from  the aligned Guassian distribution for the time series, with a fixed time delay of $\tau = 5$. To better reflect real-world measurement properties, we model finite sensor precision by restricting time series values to two decimal places. This introduces quantization effects and repeated values, which are common in practice but rarely considered in prior evaluations. We generate 10,000 time steps for both time series and temporal event sequence. The ground-truth MI is computed analytically from the known distributions using a Taylor-series approximation on the samples.


\noindent \textbf{Baselines.} 
We compare the proposed measurement against four representative baselines:  
(1)~\textbf{Series2Seq}: discretizing the time series into a symbolic sequence via binning, followed by point-wise MI computation;  
(2)~\textbf{Seq2Series}: converting the event sequence into a numeric time series using symbol IDs and measuring dependence via Pearson correlation;  
(3)~\textbf{Mixture}: a discrete–continuous MI estimator based on mixture modeling~\cite{gao2017estimating};  
and (4)~\textbf{Ross}: a nonparametric estimator for discrete–continuous MI~\cite{ross2014mutual}.  
For a fair comparison of MI estimation accuracy, we do not apply the event clustering strategy in our method for this experiment.

\noindent \textbf{Results.} 
Table~\ref{tab:tdmi} reports TDMI estimates across lag steps $\tau = 0,\ldots,9$. Transformation-based baselines Series2Seq and Seq2Series fail to recover the true dependence structure (i.e., Seq2Series has correlation less than 0.5), highlighting the limitations of forcing heterogeneous data into a homogeneous form. The Ross estimator produces negative values due to the heuristic resolution to handle the repeated values and sampling estimation variance around the 0 value, which violates the non-negativity property of MI ( Theorem~\ref{theorom:non-negative} is also valid for Ross).  

Both the Mixture estimator and our method correctly identify the peak dependence at $\tau = 5$ and near-zero dependence at other lags. However, our method consistently yields lower estimation error across different time lags $\tau$. These results demonstrate that explicitly modeling finite precision and continuous–discrete duality leads to more accurate and stable TDMI estimation. Additional synthetic experiments under varying time series lengths and value precision are reported in the Appendix~\ref{app:non-negative}. Our method consistently performs better over different data length and digit precision. 
\begin{table}[t]
\centering
\captionsetup{font=small}
\caption{
TDMI results across time lags $\tau$. 
Bold indicates the detected delay.
Values in parentheses is mean squared error (MSE, $5$ numerical precision) relative to G.T. Series2Seq and Seq2Series do not output discrete–continuous MI, so MSE is not reported. 
}
\label{tab:tdmi}
\resizebox{0.47\textwidth}{!}{
\begin{tabular}{l c c c c c c}
\toprule
$\boldsymbol{\tau}$ & G.T. & Series2Seq & Seq2Series & Ross & Mixture & Ours \\
\midrule
0 & 0.0376 & 0.0000 & -0.0158 & -3.1647 (10.254) & 0.0312 (4.0E$-$05) & 0.0353 (1.0E$-$05) \\
1 & 0.0423 & 0.0000 &  0.0095 & -3.1639 (10.279) & 0.0308 (1.3E$-$04) & 0.0352 (5.0E$-$05) \\
2 & 0.0402 & 0.0000 & -0.0071 & -3.1633 (10.262) & 0.0322 (6.0E$-$05) & 0.0373 (1.0E$-$05) \\
3 & 0.0385 & 0.0000 &  0.0024 & -3.1627 (10.247) & 0.0329 (3.0E$-$05) & 0.0387 (0) \\
4 & 0.0374 & 0.0000 &  0.0057 & -3.1628 (10.241) & 0.0314 (4.0E$-$05) & 0.0366 (0) \\
\midrule
5 & 1.0984 & 0.0000 & \textbf{-0.4998} & \textbf{-3.1468} (18.022) & \textbf{1.1078} (9.0E$-$05) & \textbf{1.0988} (0) \\
\midrule
6 & 0.0388 & 0.0000 & -0.0151 & -3.1655 (10.267) & 0.0277 (1.2E$-$04) & 0.0349 (2.0E$-$05) \\
7 & 0.0397 & 0.0000 & -0.0016 & -3.1639 (10.263) & 0.0326 (5.0E$-$05) & 0.0376 (0) \\
8 & 0.0393 & 0.0000 &  0.0006 & -3.1657 (10.272) & 0.0308 (7.0E$-$05) & 0.0358 (1.0E$-$05) \\
9 & 0.0384 & 0.0000 & -0.0063 & -3.1652 (10.263) & 0.0309 (6.0E$-$05) & 0.0355 (1.0E$-$05) \\
\bottomrule
\end{tabular}
}
\end{table}

\subsection{Repeated Temporal Patterns}
This experiment evaluates the proposed measurement for identifying \emph{repeated temporal patterns} in time series, including \textbf{global seasonality} and \textbf{local repeated patterns}. Our central contribution here is not only detecting whether a repeated pattern exists but also providing a \emph{quantitative, interpretable, and scalable} measure of pattern strength between a time series and a temporal event sequence, even when patterns are imperfect, disrupted, or context dependent. Repeated temporal patterns are often evident to humans through visual inspection when the dataset is small. For example, Figure~\ref{fig:time_steps}(a) shows a clear weekly seasonal pattern, while Figure~\ref{fig:time_steps}(b) illustrates a local repeated pattern occurring only on Wednesdays. In such limited cases, humans can manually construct an implicit ground truth by identifying whether a pattern exists and which temporal contexts (e.g., day of week) are relevant. However, this form of visual inspection does not scale to large datasets, long time horizons, or noisy real-world measurements, which are common in practice.  Instead, practical methods must (i) detect repeated temporal structure with flexibility and (ii) quantify its strength in a way that degrades gracefully under noise, missing values, and disruptions. Our proposed measurement is designed to fill this gap.

\begin{figure}
\centering
\includegraphics[width=0.9\columnwidth, height=0.08\textheight]{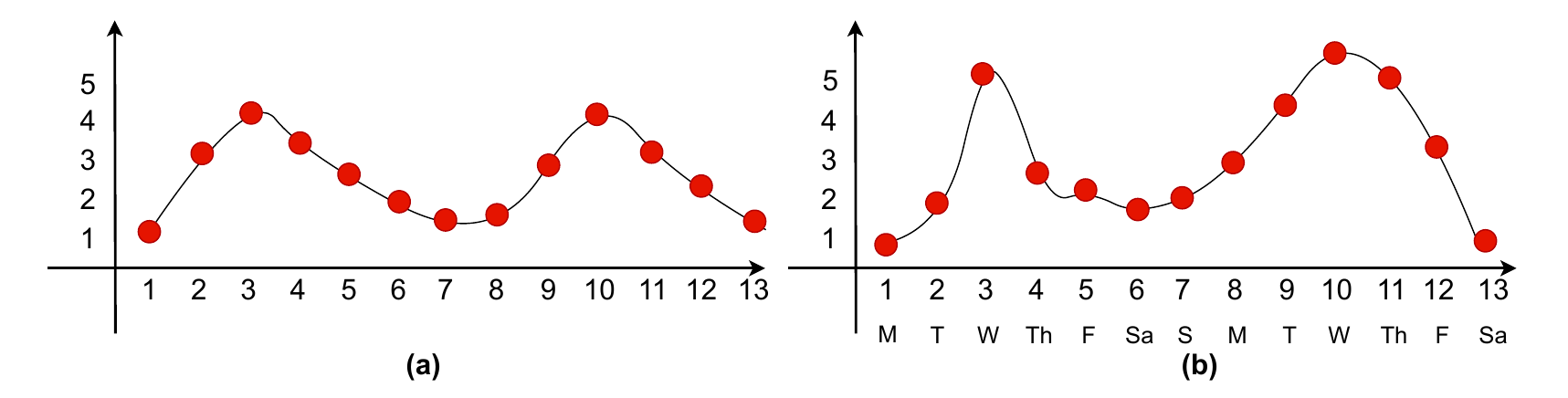}
\captionsetup{font=small}
\caption{(a) Global seasonality. (b) Local repeated pattern on Wed.}
\label{fig:time_steps}
\end{figure}

\noindent
\textbf{Seasonality} refers to recurring behaviors that repeat at fixed intervals and follow a periodic structure. By mapping timestamps to calendar attributes, we construct a temporal event sequence whose event types correspond to days of the week. Using our proposed measurement, seasonality is quantified as the MI between the time series $S$ and the calendar event sequence $W$:
{\small
\begin{displaymath}
   \text{Seasonality}(S, W) = \sum_{w \in \{\text{Mon},\dots,\text{Sun}\}} I(S, W=w).
\end{displaymath}
}
This formulation directly captures the dependence between continuous values and discrete temporal contexts, without assuming strict periodicity or requiring signal transformation. The proposed measurement can also capture \textbf{local repeated patterns} that arise only under specific temporal contexts. For example, Figure~\ref{fig:time_steps}(b) illustrates a pattern that consistently occurs on Wednesdays but not on other days. Using our framework, such a structure is naturally modeled by constructing a binary event sequence (e.g., Wednesday vs.\ non-Wednesday) and computing the MI between this event sequence and the time series, enabling principled quantification of localized and context-dependent repetition.

\subsubsection{Global Seasonality}

\noindent \textbf{Dataset.}
We use daily traffic volume data from the Department of Transportation in Minneapolis~\cite{grossman2024performance} over four months in 2023 (Figure~\ref{fig:seasonal}). Visual inspection provides a natural ground truth at this scale: strong weekly seasonality is present in June--July (Figure~\ref{fig:subfig1}), while the pattern is substantially weaker in December--January (Figure~\ref{fig:subfig2}) due to holiday disruptions.
\begin{figure}[ht!]
    \centering
    \subfloat[Traffic in Jun. and Jul.]{
        \includegraphics[width=0.23\textwidth]{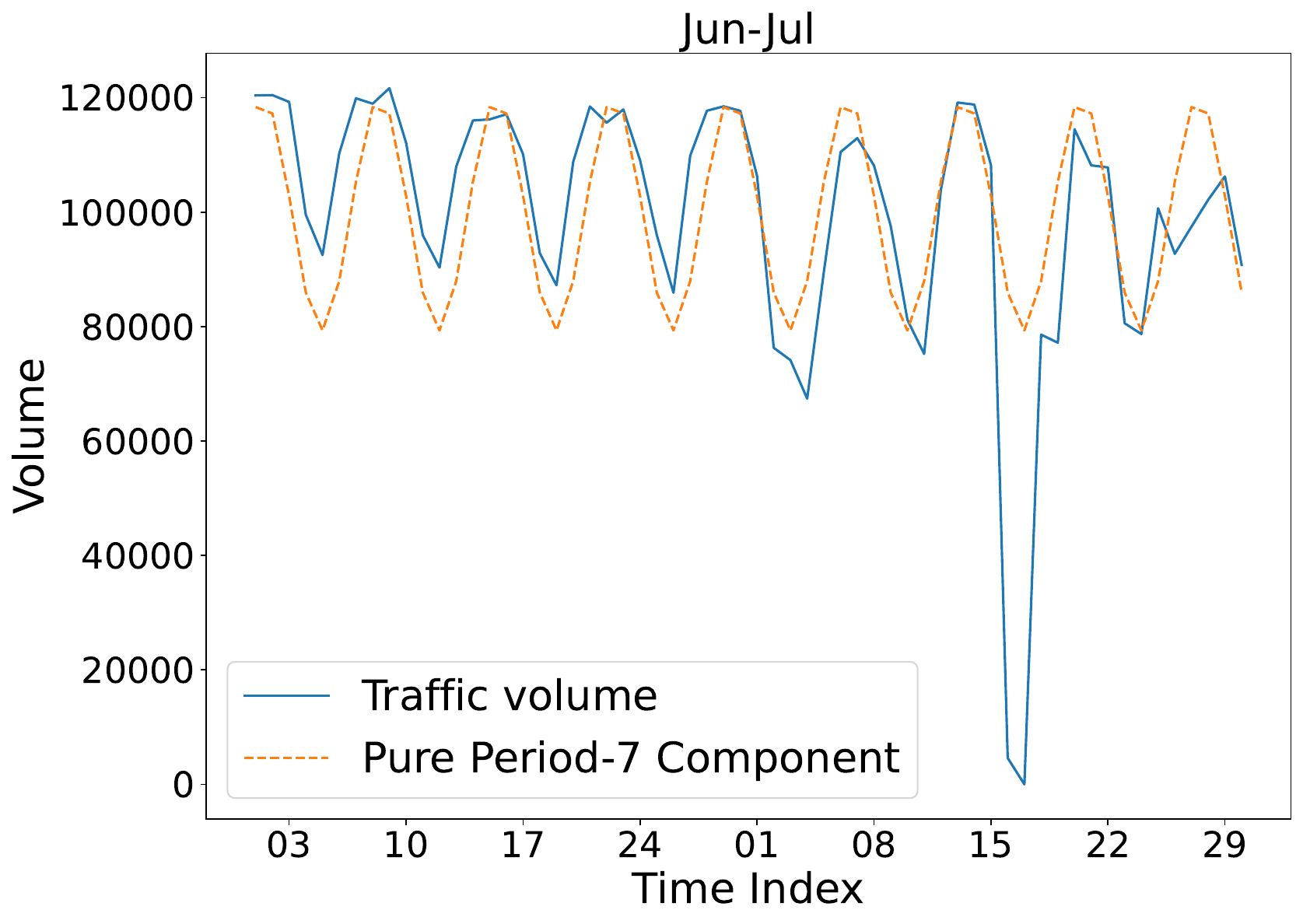}
        \label{fig:subfig1}
    }
    \subfloat[Traffic in Dec. and Jan.]{
        \includegraphics[width=0.23\textwidth]{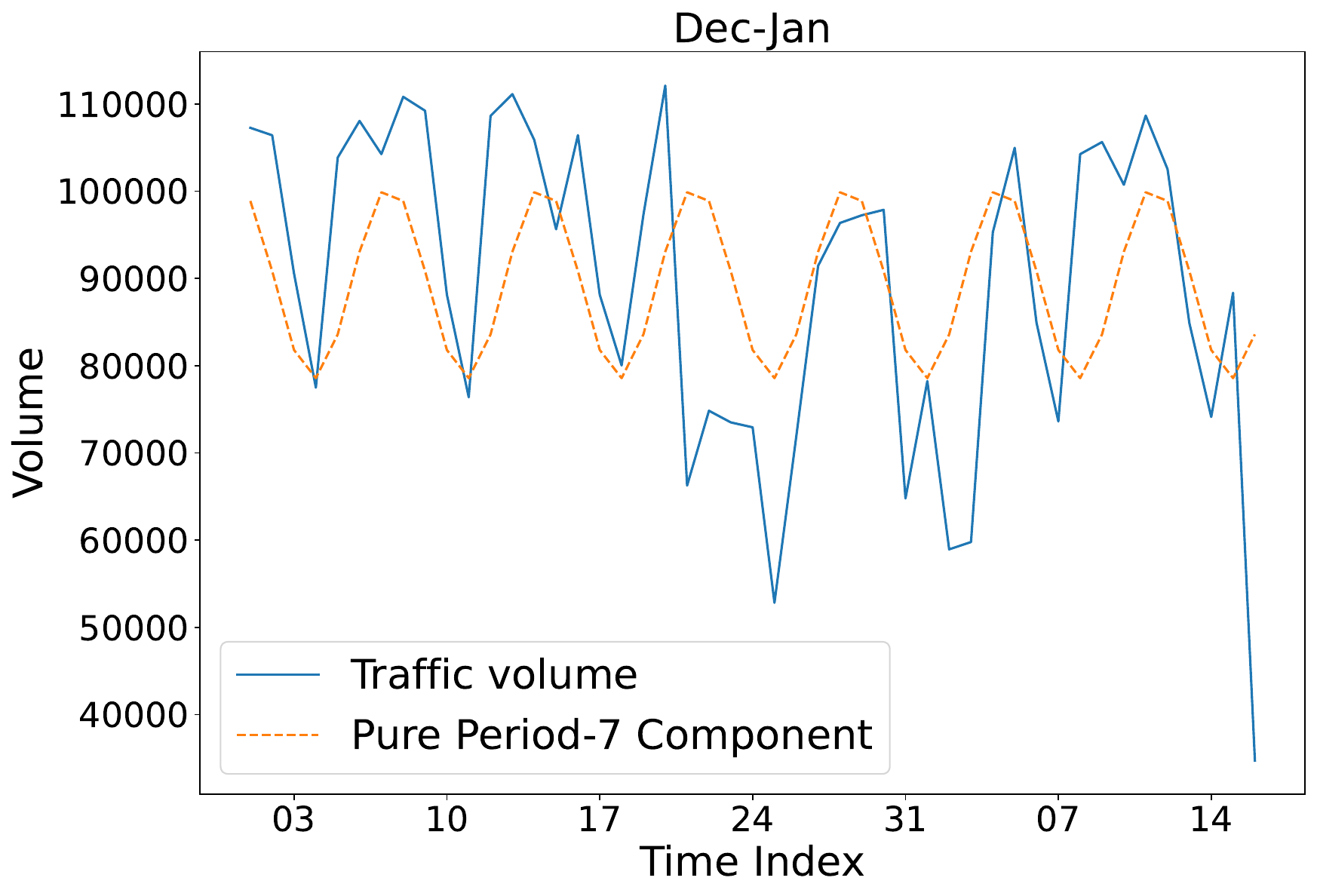}
        \label{fig:subfig2}
    }
    \captionsetup{font=small}
    \caption{Daily traffic volume in Minneapolis (blue) with weekly periodic components identified via harmonic regression (orange), which serve as a manually constructed reference for seasonal structure. Weekly seasonality is clearly presented in June--July, but disrupted in December--January due to holiday effects.}
    \label{fig:seasonal}
\end{figure}

\noindent \textbf{Baselines.}
We compare our method with classical seasonality detection approaches: (1)~\textbf{Auto-correlation function} (ACF)~\cite{box2015time}; (2)~\textbf{ACF+SD}, which applies seasonal decomposition~\cite{cleveland1990stl} followed by autocorrelation; and (3)~\textbf{Fourier Transform} (FT)~\cite{brillinger2001time}. 

\noindent \textbf{Results.}
Table~\ref{tab:seasonality} reports seasonality identification results. Both our method and SD+ACF correctly identify weekly seasonality, while ACF and FT fail under noise and irregularities. However, identifying a seasonal period alone is insufficient in practice. A useful method should also reflect \emph{how strongly} seasonality manifests and how it degrades under disruption.

 Although SD+ACF detects seasonality, it produces too high values close to 1, considering large offsets from the ideal seasonality observed in  Figure~\ref{fig:subfig1} and Figure~\ref{fig:subfig2}. In contrast, our proposed measurement yields an MI of $0.6501$ for June--July, capturing strong but imperfect seasonality, and a lower MI of $0.5664$ for December--January, reflecting stronger disruption. This shows that our method provides both detection and interpretable quantification of the strength of seasonality.

\begin{table}[th!]
\centering
\captionsetup{font=small}
\caption{Results of seasonality identification and interpretation. For FT, the dominant frequency indicates the seasonal period. For other methods, we report the maximum measure among candidate periods, with values >0.5 suggesting a likely seasonal pattern.}
\label{tab:seasonality}
\resizebox{0.3\textwidth}{!}{
\begin{tabular}{c|c|c|c}
\toprule
Data                     & Method   & Measure               & Seasonality \\ \midrule
\multirow{4}{*}{Jun-Jul} & ACF      & Autocorr=0.4651           & No          \\
                         & SD+ACF   & Autocorr=0.9817           & 7 days      \\
                         & FT       & Top 1 Freq.=0.13      & 7-8 days    \\ 
                         & \textbf{Ours}     & \textbf{MI=0.6501}             & \textbf{7 days}      \\ \midrule
\multirow{4}{*}{Dec-Jan} & ACF      & Autocorr=0.2181           & No          \\
                         & SD+ACF   & Autocorr=0.9157           & 7 days      \\
                         & FT       & Top 1 Freq.=0.08      & 11-12 days  \\ 
                         & \textbf{Ours} & \textbf{MI=0.5664}             & \textbf{7 days}      \\ \bottomrule
\end{tabular}
}
\end{table}

\subsubsection{Local Repeated Patterns}



\noindent \textbf{Dataset.}
A \emph{point repeated pattern} refers to repetition that emerges only under specific contextual conditions, rather than as a single global periodic structure. In the temperature setting, fine-grained temporal context (e.g., month and day/night) strongly constrains the temperature distribution, producing narrow, predictable value ranges, whereas coarser context yields weaker repetition and a larger temperature range. Classical wavelet-based methods focus on global periodicity and therefore struggle with such conditional, regime-dependent patterns. Our approach instead quantifies repetition through \emph{context-conditioned value distributions}, enabling a unified treatment of global seasonality, local repetition, and non-periodic structure.

We evaluate this setting using air temperature data collected at 12-hour intervals over one year~\cite{grossman2024performance}. The data are throughout 2023 (Figure~\ref{fig:temp}), resulting in 730 timestamps. Each timestamp corresponds to either the maximum daytime temperature (6\,am–6\,pm) or the minimum nighttime temperature (6\,pm–6\,am), reflecting the diurnal temperature range. Using the same timestamps, we construct three temporal event sequences with increasing contextual granularity:
(1) \textbf{DN}, labeling each timestep as \{Daytime, Night\};
(2) \textbf{TwoMon}, labeling each timestep by two-month calendar blocks (Jan–Feb,  \dots, Nov–Dec);
and (3) \textbf{DNTwoMon}, labeling each timestep by the joint context of day/night and two-month block.
For each construction, we compute the MI between the temperature series and the corresponding event sequence, measuring how much contextual information reduces uncertainty in the observed values.


\noindent \textbf{Results.}
Table~\ref{tab:point1} reports MI estimates between the temperature series and each event sequence. A valid measure of local repeated patterns should exhibit monotonic behavior as contextual information increases. Moving from DN to TwoMon to DNTwoMon introduces progressively richer temporal context that increasingly constrains the temperature distribution (e.g., January nighttime versus July daytime). Our normalized MI increases consistently (0.6124 $\rightarrow$ 0.8932 $\rightarrow$ 0.9596), matching both physical intuition and the expected reduction in conditional entropy.

In contrast, the Ross produces negative values that violate the non-negativity of MI, making it unsuitable for quantifying pattern strength in this setting. The Mixture yields positive but unbounded values whose scale grows with event granularity, making it unclear if the contextual event sequence is strong enough without knowing a max pivot. As a result, Mixture does not provide an interpretable reference in different contextual constructions.

\begin{figure}
\centering
\includegraphics[width=0.75\columnwidth, height=0.18\textheight]{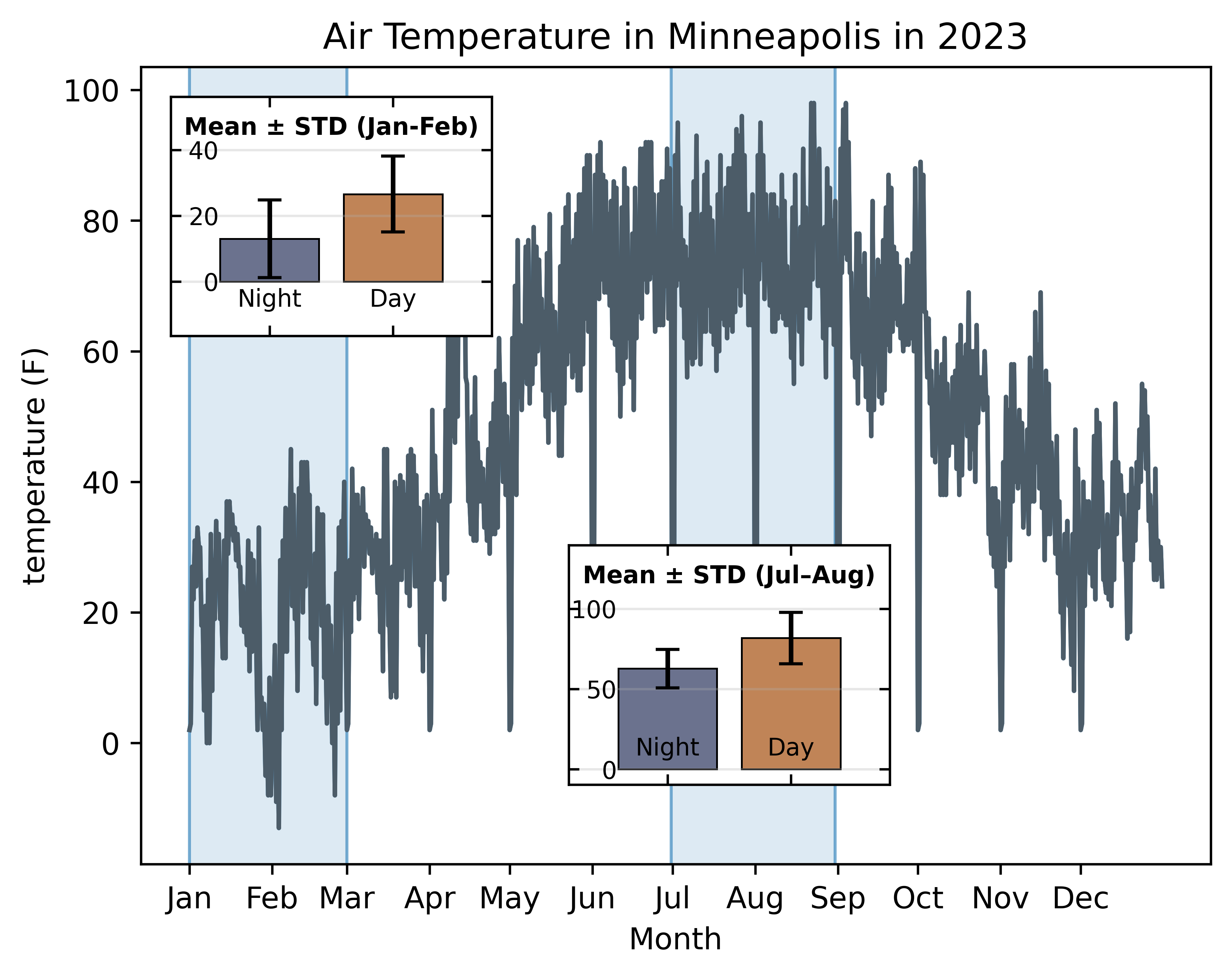}
\captionsetup{font=small}
  \caption{
Air temperature in Minneapolis during 2023, sampled at 12-hour intervals (daytime max, nighttime min). Shaded blue regions mark two-month intervals (Jan–Feb and Jul–Aug) within which daytime and nighttime temperatures form stable, narrow ranges (mean, $\pm$STD), revealing strong local repeated patterns. These patterns change across intervals, indicating that repetition is context-dependent rather than globally periodic.
}

\label{fig:temp}
\end{figure}

\begin{table}[h!]
\captionsetup{font=small}
\caption{
MI estimates for local repeated patterns under different temporal event constructions.
DN distinguishes daytime and nighttime; TwoMon partitions the year into two-month intervals; DNTwoMon combines both.
A valid measure should increase as the temporal context better aligns with temperature variation.
Our normalized MI provides a bounded, interpretable scale that clearly reflects increasing pattern strength, whereas other methods either produce invalid or unbounded scores that are difficult to interpret.
}
 \resizebox{0.65\columnwidth}{!}{%
    \begin{tabular}{l|c|c|c|c}
    \toprule
    Method  & Range & DN & TwoMon & DNTwoMon       \\ \midrule
    Ross    & $[0, \infty]$ & -1.8418 & -0.7864 & -0.2332 \\
    Mixture & $[0, \infty]$ & 0.3749 & 1.2807 & 1.5544 \\
    \textbf{Ours}    & \textbf{{[}0, 1{]}}    & \textbf{0.6124} & \textbf{0.8932} & \textbf{0.9596}    \\ \bottomrule
    \end{tabular}
\label{tab:point1}
}
\end{table}

\subsection{Discrete Covariate Identification}
This task studies discrete covariate identification for time-series forecasting with continuous-valued targets. Given a continuous target time series and a set of candidate discrete covariates, we use the proposed estimator to quantify the dependence between the target and each covariate from historical data. Covariates are then ranked by their estimated MI, and the most informative ones are selected as input to downstream forecasting models. Normalized Discounted Cumulative Gain (NDCG) is used to assess the quality of covariate rankings produced by different MI estimators. We expect that the larger the MI between the covariate and the target time series, the higher the forecast accuracy. It suggests that the covariate with higher MI could provide more useful information. Ideally, the rank of the estimated MI could serve as a proxy to forecast the accuracy rank. This experiment evaluates MI estimators through their \emph{downstream impact} on forecast performance, as ground-truth MI values are unavailable for real-world datasets. 

\noindent \textbf{Dataset.} 
We evaluate this task using the \textbf{Rossmann Store Sales dataset}\footnote{\small \url{https://www.kaggle.com/competitions/rossmann-store-sales}} and the \textbf{M5 Forecasting dataset}\footnote{\small \url{https://www.kaggle.com/competitions/m5-forecasting-accuracy}}. The Rossmann dataset contains store-level daily sales time series with associated discrete covariates such as promotions and holidays (Figure~\ref{fig:covariate_selection1}). The M5 dataset consists of daily unit sales for Walmart products, along with categorical attributes such as product, department, and calendar-related variables. In both datasets, the goal is to identify informative discrete covariates for forecasting continuous daily sales values. We randomly sample 100 time series from each dataset. Dataset details are provided in the Appendix~\ref{app:forecasting}.

\begin{figure}[h!]
\centering
\includegraphics[width=0.6\columnwidth, height=0.16\textheight]{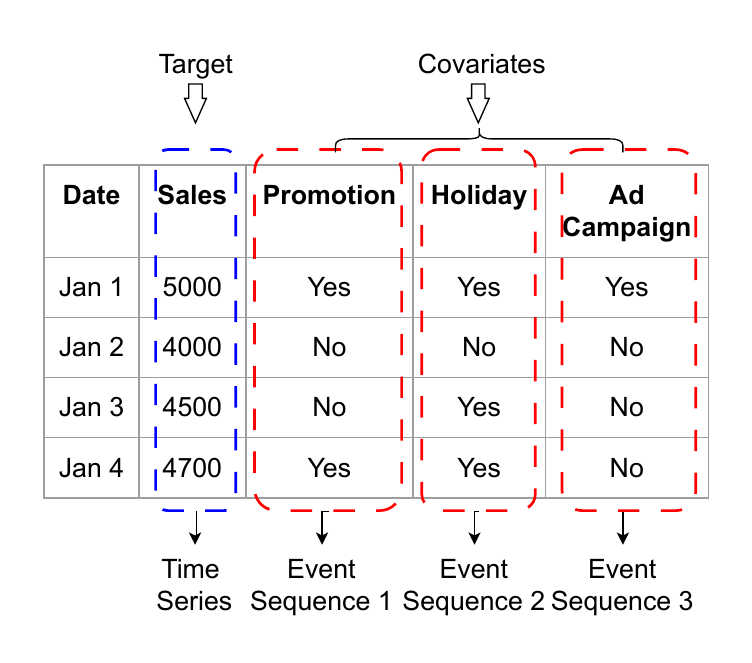}
\captionsetup{font=small}
\caption{Illustration of discrete covariate identification for forecasting. The sales time series ({\color{blue}{blue}}) is the prediction target, and each discrete covariate ({\color{red}{red}}) is modeled as a temporal event sequence.}
\label{fig:covariate_selection1}
\end{figure}

\noindent \textbf{Baselines \& Forecasting Methods.} We compare our proposed estimator with \textbf{Ross}~\cite{ross2014mutual} and \textbf{Mixture}~\cite{gao2017estimating} for MI estimation and covariate ranking. The selected covariates are then used as inputs to three representative forecasting models: CatBoostRegressor~\cite{prokhorenkova2018catboost}, DeepAR~\cite{salinas2020deepar},  FMTimeSeries~\cite{das2024a}, and Chronos-2~\cite{ansari2025chronos}. We also evaluate an ablated variant of our method without event clustering, denoted as \textbf{Ours}, and a full variant with clustering, denoted as \textbf{Ours-Cluster}. This comparison isolates the effect of grouping correlated discrete values into latent categories during covariate selection.

\noindent \textbf{Results.}Table~\ref{tab:timeFM} reports forecasting performance measured by mean NDCG over 100 sampled time series and the number of times that each method achieves the highest NDCG. Since higher MI indicates stronger and more reliable dependence between covariates and the target series, more accurate MI estimation is expected to yield better covariate selection and better forecasting performance.

Across both datasets and all models, our proposed method consistently outperforms Ross and Mixture in terms of mean NDCG and the number of winning rounds. This improvement is observed consistently across tree-based (CatBoost), probabilistic (DeepAR), and foundation-model-based (FMTimeSeries and Chronos-2) approaches, indicating that the benefit arises from improved covariate selection rather than from model-specific effects. Moreover, \textbf{Ours-Cluster} achieves the highest mean NDCG and the most winning rounds in nearly all settings. This result highlights the importance of modeling correlations and redundancies among discrete covariates by grouping them into latent categories, further improving the stability and effectiveness of MI-based covariate selection.

\begin{table}[h!]
\centering
\captionsetup{font=small}
\caption{Results of covariate selection for  forecasting. Higher NDCG indicates better MI estimation. Values in (parentheses) denote the number of winning rounds (highest NDCG) out of 100 samples.}
\label{tab:timeFM}
\resizebox{\columnwidth}{!}{
\begin{tabular}{l|c|c|c|c|c}
\toprule
Model                              & Benchmark & Ross & Mixure & Ours & \textbf{Ours-Cluster}                          \\ \midrule
\multirow{2}{*}{CatBoostRegressor} & Rossmann  & 0.85 (4)    & 0.75 (5)    & 0.94 (29)     & \textbf{0.95} (62) \\ 
                                   & M5        & 0.75 (17)   & 0.62 (5)    & 0.81 (15)      & \textbf{0.86} (63) \\ \midrule
\multirow{2}{*}{DeepAR}            & Rossmann  & 0.78 (2)    & 0.76 (10)    & 0.90 (23)     & \textbf{0.92} (65) \\ 
                                   & M5        & 0.71 (17)   & 0.54 (4)    & 0.76 (3)      & \textbf{0.82} (76) \\ \midrule
\multirow{2}{*}{FMTimeSeries}      & Rossmann  & 0.82 (0)    & 0.92 (28)   & 0.93 (0)      & \textbf{0.96} (72) \\ 
                                   & M5        & 0.88 (27)   & 0.74 (18)   & 0.89 (8)     & \textbf{0.92} (47) \\
                                   \midrule
\multirow{2}{*}{Chronos-2}      & Rossmann  & 0.72 (0)    & 0.77 (2)   & 0.92 (1)      & \textbf{0.95} (97) \\ 

                                   & M5        & 0.70 (9)	   & 0.59 (7)   & 0.78 (11)     & \textbf{0.83} (73) \\ \bottomrule
\end{tabular}
}
\end{table}

\subsection{Continuous Feature Selection} 
We further evaluate our estimator on continuous feature selection for mixed-type tabular data, where input features are continuous variables and the target label is discrete. Given a continuous feature, we treat its values across all instances as samples from a continuous random variable, and the corresponding class labels as a discrete random variable. This setting corresponds to the canonical use case of   \textbf{Ross}~\cite{ross2014mutual} and \textbf{Mixture}~\cite{gao2017estimating}.

Our proposed estimator directly computes the MI between a continuous feature and a discrete label without distributional assumptions or ad hoc discretization. Although originally motivated by temporal data, our formulation naturally generalizes to tabular settings by aggregating entropy with permutation-invariant operators (i.e., summation and integration), making feature relevance independent of instance ordering. Conceptually, this corresponds to modeling feature values as samples drawn from an underlying stochastic process, an assumption that is equally valid for both time series and i.i.d.\ tabular data.

Based on this formulation, we define the feature importance (FI) between a continuous feature variable $F$ and a discrete label variable $L$ as
{\small
$
    \text{FI}(L, F) = \sum_{l \in \mathcal{L}}\int_{f\in \mathcal{F}}P(l,f)\log \left(\frac{P(l, f)}{P(l)P(f)}\right) \;df.
$
}

\noindent \textbf{Dataset.} We evaluate feature selection performance using datasets from a widely adopted benchmark~\cite{li2018feature}, which includes multiple classification datasets with continuous features and discrete class labels across diverse domains such as image analysis, biological data, and spoken letter recognition. Additional dataset details are provided in the Appendix~\ref{app:feature_selection}.

\noindent \textbf{Baselines \& Classifiers.} We compare our estimator with \textbf{Ross}~\cite{ross2014mutual} and \textbf{Mixture}~\cite{gao2017estimating} for MI-based feature selection. For all methods, the top-$k$ selected features are used as inputs to three standard classifiers: Random Forest (RF), Support Vector Machine (SVM), and Logistic Regression (LR).

\noindent \textbf{Results.} We evaluate feature relevance through downstream classification performance. Detailed classification accuracies using the top $k$ selected features are reported in Detailed classification results for individual datasets and feature budgets are reported in Tables~\ref{tab:compare_feature} and  \ref{tab:compare_feature1} in the Appendix. Table~\ref{tab:win_tie_lose} summarizes performance over a total of 120 evaluation cases, spanning all datasets, classifiers, and feature budgets. In each case, a \emph{win} indicates that features selected by our method yield higher classification accuracy than those selected by a baseline method for the same classifier and dataset; \emph{tie} and \emph{loss} are defined analogously. Across all settings, this aggregated comparison highlights the consistency of our method relative to existing approaches. Our proposed measurement consistently outperforms existing MI-based methods, indicating more accurate quantification of dependence between continuous features and discrete labels. Notably, the optional clustering strategy further improves performance by grouping redundant or highly overlapping feature-value distributions into latent categories, reducing estimation noise and enhancing feature ranking stability.

                              

\begin{table}[ht!]
\centering
\captionsetup{font=small}
\caption{ The count of wins, ties, and loses across 120 settings.}
\label{tab:win_tie_lose}
 \resizebox{0.58\columnwidth}{!}{%
\begin{tabular}{l|c|c|c}
\toprule
Method Comparison & Win & Tie  & Lose  \\ \midrule
Ours VS. Ross      & 89        & 14        & 17  \\ \midrule
Ours VS. Mixtures  & 93   & 15   &12\\ \midrule
W/ cluster VS. W/O cluster  &77   &20   &23 \\ \bottomrule
\end{tabular}
}
\end{table}

\section{Discussion}
We introduced a nonparametric mutual-information–based measurement for directly quantifying interactions between heterogeneous temporal data, specifically continuous time series and discrete event sequences. Unlike prior approaches, the proposed estimator operates natively on mixed data types without requiring discretization, distributional assumptions, learned models, or ad hoc data transformations. This design allows the measurement to remain simple, interpretable, and theoretically grounded while addressing practical issues such as finite precision, repeated values, and correlated events that commonly arise in real-world temporal data. Across a diverse set of fundamental analysis tasks, our experiments demonstrate that the proposed measurement consistently achieves the intended analytical objectives and often outperforms existing discrete–continuous MI estimators and classical correlation-based methods. These results suggest that the estimator can serve as a reliable, task-agnostic measure of dependence between heterogeneous temporal variables, analogous in spirit to Pearson correlation but applicable to mixed data types and nonlinear relationships. Future work includes extending the estimator to better accommodate application-specific noise characteristics and integrating it with learning-based frameworks that can further improve robustness under severe noise or sparsity.


\begin{acks}
We thank the anonymous reviewers for their constructive feedback. Specifically, thanks to Junhao Zhu for the insightful discussion. 
This work has been funded in part by the NIH award R01LM014026.
\end{acks}


\bibliographystyle{ACM-Reference-Format}
\bibliography{reference}

\appendix
\section{Proof of the Measurement's Non-negativity}
\label{app:non-negative}

The classical Shannon mutual information (i.e., the mutual information between two discrete variables) is non-negative, and the mutual information between two continuous variables is also proved to be non-negative~\cite {nagel2024accurate}. We provide the corresponding proof for the mutual information between a continuous variable and a discrete variable. This is the foundation of the forced non-negativity operation in our estimated mutual information.

\begin{theorem}
\label{theorom:non-negative}
The mutual information between a continuous variable and a discrete variable is non-negative. 
\end{theorem}

\noindent
\textbf{Proof:}

As $H(X)\geq 0$, $\hat{H}(Y)\geq 0$ and $\small \min(H(X),\hat{H}(Y))\geq \hat{H}(X,Y)\geq 0$, $\small H(X) + \hat{H}(Y) - \hat{H}(X, Y)\geq 0$, 
{\small
\begin{align}
\I(X, Y) &= H(X) + \hat{H}(Y) - \hat{H}(X, Y)\\
&= -\sum_{x\in \mathcal{X}} P(x)\log P(x) - \int_{\mathcal{Y}}P(y)\log \frac{P(y)}{m(y)}\D y \\ 
&+ \sum_{x\in \mathcal{X}} \int_{\mathcal{Y}}P(x,y)\log \frac{P(x,y)}{m(y)} \D y \\
&= -\sum_{x\in \mathcal{X}} P(x)\log P(x) - \int_{\mathcal{Y}}P(y)\log P(y)\D y \\
& + \int_{\mathcal{Y}}P(y)\log m(y)\D y  
 +\sum_{x\in \mathcal{X}} \int_{\mathcal{Y}}P(x,y)\log P(x,y)\D y \\
 & - \sum_{x\in \mathcal{X}} \int_{\mathcal{Y}}P(x,y)\log m(y) \D y \\
&= H(X) + H(Y) - H(X, Y) \geq 0
\label{eq:mi_relative}
\end{align}
}

\section{Theoretical Analysis Discussion}

Existing theoretical analysis on the estimator, like consistency analysis and finite-sample bounds, usually relies on the assumption of a large sample size. However, real-world time series have limited length, e.g., tens to thousands. Those bounds, in the worst case, are loose and lack practical guidance. In addition, consistency analysis for clustering is an open challenge~\cite{blanchard2025consistency}. Therefore, we rely on empirical evaluation using real-world data rather than error bound analysis. Therefore, we rely on diverse empirical evaluations from real-world data.

\section{Mutual Information Estimation Algorithm}
\label{app:continuous_estimation}

Our estimation algorithm could be formalized as Algorithm~\ref{alg:mi}. The algorithm to estimate the continuous entropy is Algorithm~\ref{app:algo}.

\begin{algorithm}
\caption{Normalized mutual information calculation\\\textbf{Input}: A temporal event sequence $X$ and a time series $Y$.} 
\label{alg:mi}
\begin{algorithmic}[1]
\small
\State Get the symbol set $X_{set}$ from $X$
\State Assign all values aligning with a specific temporal event $e$ over time by group[e] = [list contains continuous values] 
\If {correlation among events is observed}
\State Use $hierarchical\_clustering$($group$, $num\_clusters$) to remap events to latent cluster events and generate $X_{new}$
\State $X = X_{new}$ 
\EndIf
\State Partition $(X, Y)$ based on repeated values into $(X_A,Y_A)$ and $(X_B, Y_B)$.
\State $a =  \frac{\|(X_A,Y_A)\|}{\|(X, Y)\|}$, $b =  \frac{\|(X_B,Y_B)\|}{\|(X, Y)\|}$
\State $\I_A = continuous\_entropy(Y_A)$
\For {each event symbol $e$ in $X_{set}$}
\State $\I_A -= \frac{\len(group[e])}{\len(Y)} continuous\_entropy(group_A[e])$
\EndFor
\State Calculate  $\I_B$ with Eq.~\ref{eq:mi-dis}.
\State $\I = a\cdot \I_A+b\cdot \I_B + a\cdot ln(a) + b\cdot ln(b)$
\State Clip the negative estimation $\I=\max(\I, 0)$
\State Calculate $\NMI(X,Y)$ (with Eq.~\ref{eq:nomralization})
\State \textbf{return} $\NMI(X,Y)$
\Statex 
\end{algorithmic} 
\end{algorithm}

\begin{algorithm}
\caption{continuous\_entropy\\
\textbf{Input}: A time series as variable $Y$}
\label{app:algo}
\begin{algorithmic}[1]
\small

\State Sort all values in $Y$ and get $sorted\_Y$
\State Get the difference between all the consecutive pairs in $sorted\_Y$
\State Initialize an array $\rho$, each element corresponds to a value $v_i$ in $sorted\_Y$

\For{each $v_i \in sorted\_Y$}
  \State Calculate the nearest neighbor distance:
  \State $\rho_i = \min(v_i - v_{i-1},\, v_{i+1} - v_i)$
\EndFor

\State $sum \gets 0$
\For{each $\rho_i$}
  \State $sum = sum + \log(\rho_i)$
\EndFor

\State $mean\_\rho = \frac{sum}{|\rho|}$
\State $H(Y) = mean\_\rho + \ln(c) + \ln(\gamma) + \ln(|\rho| - 1)$
\State \textbf{return} $H(Y)$
\Statex
\end{algorithmic}
\end{algorithm}


\section{Time-Delayed Mutual Information in Different Settings}

The synthetic data distribution is visualized in Fig.~\ref{fig:covariate_selection}.

\begin{figure}
\centering
\includegraphics[width=0.63\columnwidth, height=0.14\textheight]{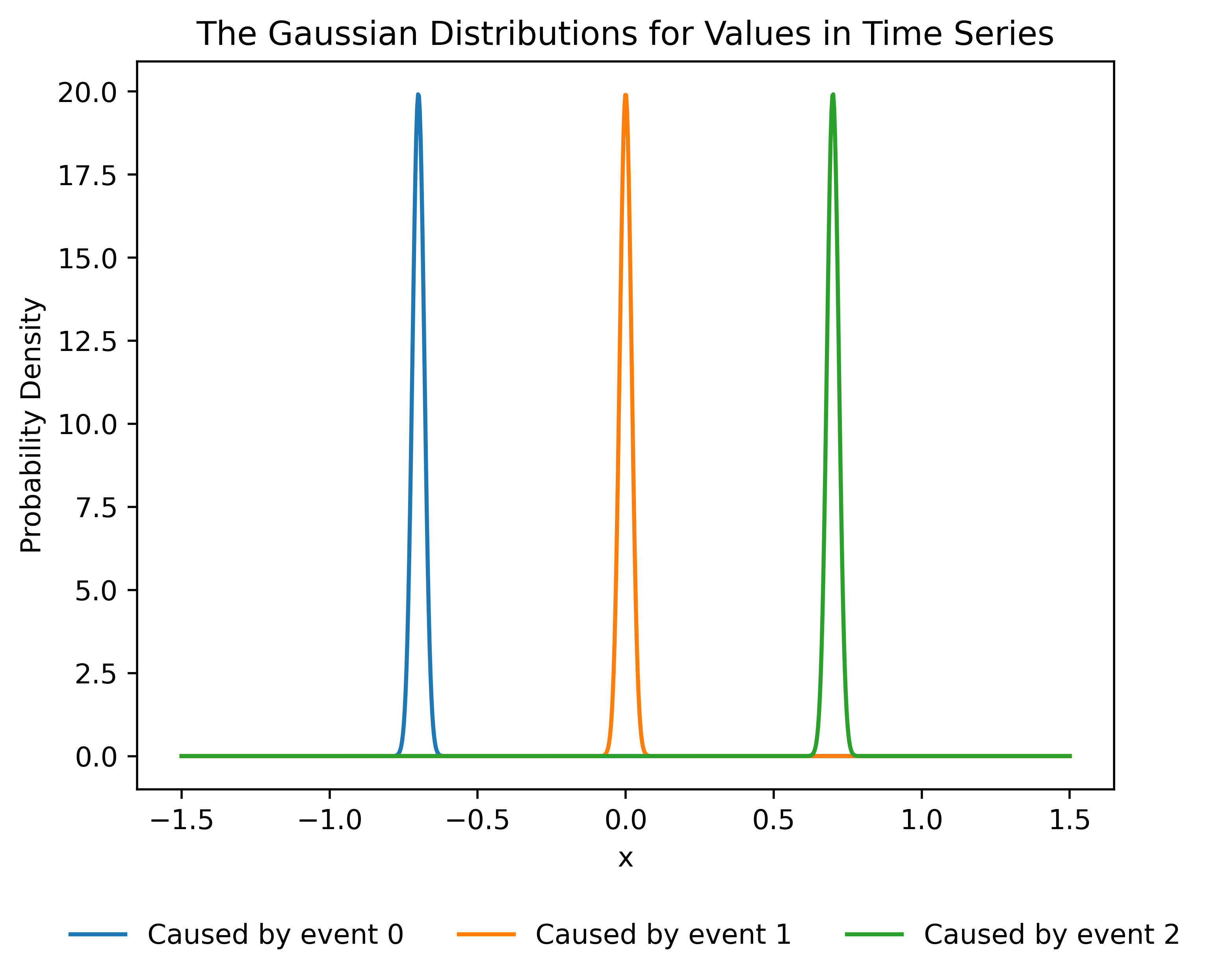}
\captionsetup{font=small}
  \caption{ The Gaussian distributions for time series.}
\label{fig:covariate_selection}
\end{figure}

Table~\ref{tab:tdmi} shows the results when the decimal place number is 3 and the time series length is 10,000. To better demonstrate the robust performance of the proposed method, we compare the proposed measurement with the Ross and Mixture methods across different digit precision settings (Table~\ref{tab:tdmi_digitnum}) and sample numbers (Table~\ref{tab:tdmi_sample_num}).

Table~\ref{tab:tdmi_digitnum} suggested that the proposed measurement can achieve the best accuracy consistently with different digit numbers. This demonstrates the robustness of our measurement.

Table~\ref{tab:tdmi_sample_num} suggested that the proposed measurement could consistently perform well even with a few samples.

\begin{table}[]
\centering
\captionsetup{font=small}
\caption{ The comparison with 10,000 time step samples under different decimal places precision. Only report the mean square error (rounded to 5 decimal places) for simplicity.}
\label{tab:tdmi_digitnum}
\resizebox{1\columnwidth}{!}{%
\begin{tabular}{|l|ccc|ccc|ccc|}
\hline
  & \multicolumn{3}{c|}{Digit Num. = 2}                                                 & \multicolumn{3}{c|}{Digit Num. = 3}                                                        & \multicolumn{3}{c|}{Digit Num. = 4}                                                        \\ \hline
 $\tau$ & \multicolumn{1}{c|}{Ro. MSE}         & \multicolumn{1}{c|}{Mi. MSE}    & Pr. MSE    & \multicolumn{1}{c|}{Ro. MSE}         & \multicolumn{1}{c|}{Mi. MSE}           & Pr. MSE    & \multicolumn{1}{c|}{Ro. MSE} & \multicolumn{1}{c|}{Mi. MSE} & \multicolumn{1}{c|}{Pr. MSE} \\ \hline
0 & \multicolumn{1}{c|}{29.993}          & \multicolumn{1}{c|}{0}          & 0          & \multicolumn{1}{c|}{10.254}          & \multicolumn{1}{c|}{4.00E-05}          & 1.00E-05   & \multicolumn{1}{c|}{1.1882}  & \multicolumn{1}{l|}{0.1270}  & 0.0468                       \\ \hline
1 & \multicolumn{1}{c|}{30.000}          & \multicolumn{1}{c|}{0}          & 0          & \multicolumn{1}{c|}{10.279}          & \multicolumn{1}{c|}{0.00013}           & 5.00E-05   & \multicolumn{1}{c|}{1.1649}  & \multicolumn{1}{l|}{0.1374}  & 0.0482                       \\ \hline
2 & \multicolumn{1}{c|}{30.002}          & \multicolumn{1}{c|}{0}          & 0          & \multicolumn{1}{c|}{10.262}          & \multicolumn{1}{c|}{6.00E-05}          & 1.00E-05   & \multicolumn{1}{c|}{1.1892}  & \multicolumn{1}{l|}{0.1373}  & 0.0474                       \\ \hline
3 & \multicolumn{1}{c|}{29.990}          & \multicolumn{1}{c|}{0}          & 0          & \multicolumn{1}{c|}{10.247}          & \multicolumn{1}{c|}{3.00E-05}          & 0          & \multicolumn{1}{c|}{1.1834}  & \multicolumn{1}{l|}{0.1373}  & 0.0488                       \\ \hline
4 & \multicolumn{1}{c|}{29.992}          & \multicolumn{1}{c|}{0}          & 0          & \multicolumn{1}{c|}{10.241}          & \multicolumn{1}{c|}{4.00E-05}          & 0          & \multicolumn{1}{c|}{1.2047}  & \multicolumn{1}{l|}{0.1364}  & 0.0450                       \\ \hline
5 & \multicolumn{1}{c|}{\textbf{43.153}} & \multicolumn{1}{c|}{\textbf{0}} & \textbf{0} & \multicolumn{1}{c|}{\textbf{18.022}} & \multicolumn{1}{c|}{\textbf{9.00E-05}} & \textbf{0} & \multicolumn{1}{c|}{3.2257}  & \multicolumn{1}{l|}{0.0293}  & 0                            \\ \hline
6 & \multicolumn{1}{c|}{30.003}          & \multicolumn{1}{c|}{0}          & 0          & \multicolumn{1}{c|}{10.267}          & \multicolumn{1}{c|}{0.00012}           & 2.00E-05   & \multicolumn{1}{c|}{1.1939}  & \multicolumn{1}{l|}{0.1443}  & 0.0456                       \\ \hline
7 & \multicolumn{1}{c|}{29.994}          & \multicolumn{1}{c|}{0}          & 0          & \multicolumn{1}{c|}{10.263}          & \multicolumn{1}{c|}{5.00E-05}          & 0          & \multicolumn{1}{c|}{1.1689}  & \multicolumn{1}{l|}{0.1401}  & 0.0477                       \\ \hline
8 & \multicolumn{1}{c|}{30.000}          & \multicolumn{1}{c|}{0}          & 0          & \multicolumn{1}{c|}{10.272}          & \multicolumn{1}{c|}{7.00E-05}          & 1.00E-05   & \multicolumn{1}{c|}{1.2040}  & \multicolumn{1}{l|}{0.1362}  & 0.0445                       \\ \hline
9 & \multicolumn{1}{c|}{29.993}          & \multicolumn{1}{c|}{0}          & 0          & \multicolumn{1}{c|}{10.263}          & \multicolumn{1}{c|}{6.00E-05}          & 1.00E-05   & \multicolumn{1}{c|}{1.2072}  & \multicolumn{1}{l|}{0.1313}  & 0.045                        \\ \hline
\end{tabular}
}
\end{table}

\begin{table}[]
\centering
\captionsetup{font=small}
\caption{ The comparison with time series in two decimal places under sample time step numbers. Only report the mean square error (rounded to 5 decimal places) for simplicity.}
\label{tab:tdmi_sample_num}
\resizebox{1\columnwidth}{!}{%
\begin{tabular}{|l|ccc|ccc|ccc|}
\hline
  & \multicolumn{3}{c|}{Time Step Num. = 100}                                                                    & \multicolumn{3}{c|}{Time Step Num. = 1000}                                                                   & \multicolumn{3}{c|}{Time Step Num. = 10000}                                            \\ \hline
  & \multicolumn{1}{c|}{Ro. MSE}        & \multicolumn{1}{c|}{Mi. MSE}         & \multicolumn{1}{c|}{Pr. MSE} & \multicolumn{1}{c|}{Ro. MSE}        & \multicolumn{1}{c|}{Mi. MSE}         & \multicolumn{1}{c|}{Pr. MSE} & \multicolumn{1}{c|}{Ro. MSE}         & \multicolumn{1}{c|}{Mi. MSE}    & Pr. MSE    \\ \hline
0 & \multicolumn{1}{l|}{1.152}          & \multicolumn{1}{l|}{0.1319}          & 0.0153                       & \multicolumn{1}{l|}{1.182}          & \multicolumn{1}{l|}{0.1741}          & 0.0556                       & \multicolumn{1}{c|}{29.993}          & \multicolumn{1}{c|}{0}          & 0          \\ \hline
1 & \multicolumn{1}{l|}{1.234}          & \multicolumn{1}{l|}{0.0844}          & 0.04248                      & \multicolumn{1}{l|}{1.156}          & \multicolumn{1}{l|}{0.1691}          & 0.0679                       & \multicolumn{1}{c|}{30.000}          & \multicolumn{1}{c|}{0}          & 0          \\ \hline
2 & \multicolumn{1}{l|}{1.265}          & \multicolumn{1}{l|}{0.1480}          & 0.00169                      & \multicolumn{1}{l|}{1.159}          & \multicolumn{1}{l|}{0.1857}          & 0.0756                       & \multicolumn{1}{c|}{30.002}          & \multicolumn{1}{c|}{0}          & 0          \\ \hline
3 & \multicolumn{1}{l|}{1.123}          & \multicolumn{1}{l|}{0.1410}          & 0.01668                      & \multicolumn{1}{l|}{1.198}          & \multicolumn{1}{l|}{0.1628}          & 0.0600                       & \multicolumn{1}{c|}{29.990}          & \multicolumn{1}{c|}{0}          & 0          \\ \hline
4 & \multicolumn{1}{l|}{1.106}          & \multicolumn{1}{l|}{0.1871}          & 0.01358                      & \multicolumn{1}{l|}{1.146}          & \multicolumn{1}{l|}{0.1836}          & 0.0843                       & \multicolumn{1}{c|}{29.992}          & \multicolumn{1}{c|}{0}          & 0          \\ \hline
5 & \multicolumn{1}{l|}{\textbf{3.320}} & \multicolumn{1}{l|}{\textbf{0.0722}} & \textbf{0.00047}             & \multicolumn{1}{l|}{\textbf{3.317}} & \multicolumn{1}{l|}{\textbf{0.0371}} & \textbf{0}                   & \multicolumn{1}{c|}{\textbf{43.153}} & \multicolumn{1}{c|}{\textbf{0}} & \textbf{0} \\ \hline
6 & \multicolumn{1}{l|}{1.054}          & \multicolumn{1}{l|}{0.1061}          & 0.00472                      & \multicolumn{1}{l|}{1.131}          & \multicolumn{1}{l|}{0.1768}          & 0.0667                       & \multicolumn{1}{c|}{30.003}          & \multicolumn{1}{c|}{0}          & 0          \\ \hline
7 & \multicolumn{1}{l|}{0.956}          & \multicolumn{1}{l|}{0.0622}          & 0.07345                      & \multicolumn{1}{l|}{1.108}          & \multicolumn{1}{l|}{0.1727}          & 0.0635                       & \multicolumn{1}{c|}{29.994}          & \multicolumn{1}{c|}{0}          & 0          \\ \hline
8 & \multicolumn{1}{l|}{1.179}          & \multicolumn{1}{l|}{0.1852}          & 0.02533                      & \multicolumn{1}{l|}{1.175}          & \multicolumn{1}{l|}{0.1829}          & 0.0553                       & \multicolumn{1}{c|}{30.000}          & \multicolumn{1}{c|}{0}          & 0          \\ \hline
9 & \multicolumn{1}{l|}{0.987}          & \multicolumn{1}{l|}{0.2169}          & 0.0116                       & \multicolumn{1}{l|}{1.078}          & \multicolumn{1}{l|}{0.1692}          & 0.0696                       & \multicolumn{1}{c|}{29.993}          & \multicolumn{1}{c|}{0}          & 0          \\ \hline
\end{tabular}
}
\end{table}


\section{Cut Threshold in Hierarchical Clustering}

The classical agglomerative clustering method~\cite{pedregosa2011scikit} has a distance threshold parameter defining the linkage distance at which the clustering process stops merging clusters. We only use the clustering strategy in the covariate selection and feature selection tasks. We fix this threshold to 0.9 for all experiments for simplicity. We also illustrate the sensitivity of this parameter in Figure~\ref{fig:clustering_threshold} for the discrete covariate selection task in the M5 dataset. 
\begin{figure}[]
    \centering
    \subfloat{
        \includegraphics[width=0.28\textwidth]{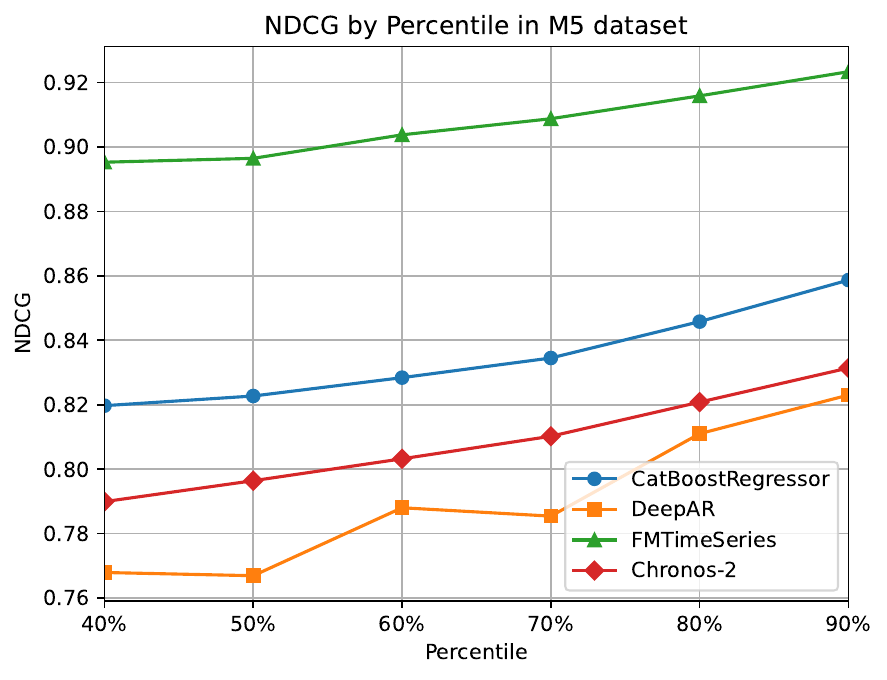}
        \label{fig:subfig4}
    }
    \\
    \subfloat{
        \includegraphics[width=0.28\textwidth]{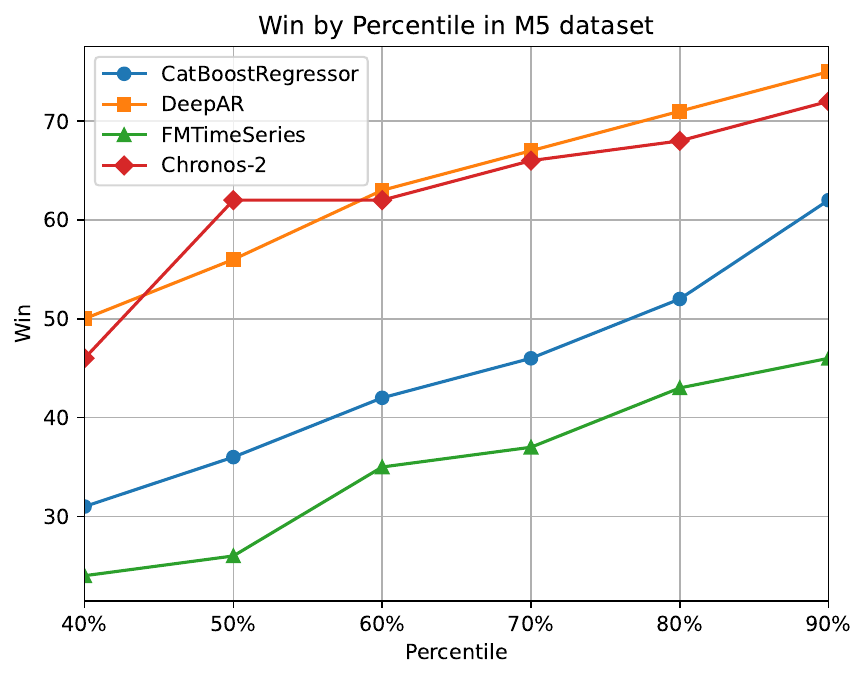}
        \label{fig:subfig5}
    }
    \captionsetup{font=small}
    \caption{Distance Threshold Sensitivity}
    \label{fig:clustering_threshold}
\end{figure}

\section{Setting for Discrete Covariate Selection }
\label{app:forecasting}



For the Rossmann Store Sales dataset, the goal is to forecast the daily sales number time series from Rossmann drug stores with the following discrete covariates:
\begin{itemize}
    \item{Open} - 0 = store is closed, 1 = store is open.
    \item{DayOfWeek} - an indicator for the day of a week.
    \item{Promo} - indicates if a store is running a promo.
    \item{StateHoliday} - indicates a state holiday. Normally, all stores are closed on state holidays. Note that all schools are closed on public holidays and weekends. a = public holiday, b = Easter holiday, c = Christmas, 0 = None.
\end{itemize}

For the M5 dataset, the goal is to forecast the daily sales number time series of one retail good with the following discrete covariates:
\begin{itemize}
    \item{Weekday} - an indicator for the day of the week.
    \item{Event name 1} - indicates the first event occurs on a specific day. It has various events.
    \item{Event type 1} - indicates the event types of the first event.
    \item{Event name 2} - indicates the second event occurs on a specific day. 
    \item{Event type 2} - indicates the event types of the second event.
    \item{Snap CA} - indicates whether the stores of CA allow SNAP purchases on the date examined. 1 indicates allowed. 
    \item{Snap TX} - indicates whether the stores of TX allow SNAP purchases on the date examined. 1 indicates allowed.
    \item{Snap WI} - indicates whether the stores of WI allow SNAP purchases on the date examined. 1 indicates allowed.
\end{itemize}

\section{Setting and Results for Feature Selection}
\label{app:feature_selection}

A toy example of feature selection could be found in Fig.~\ref{fig:feature_selection}.

\begin{figure}[h]
\centering
\includegraphics[width=0.72\columnwidth, height=0.17\textheight]{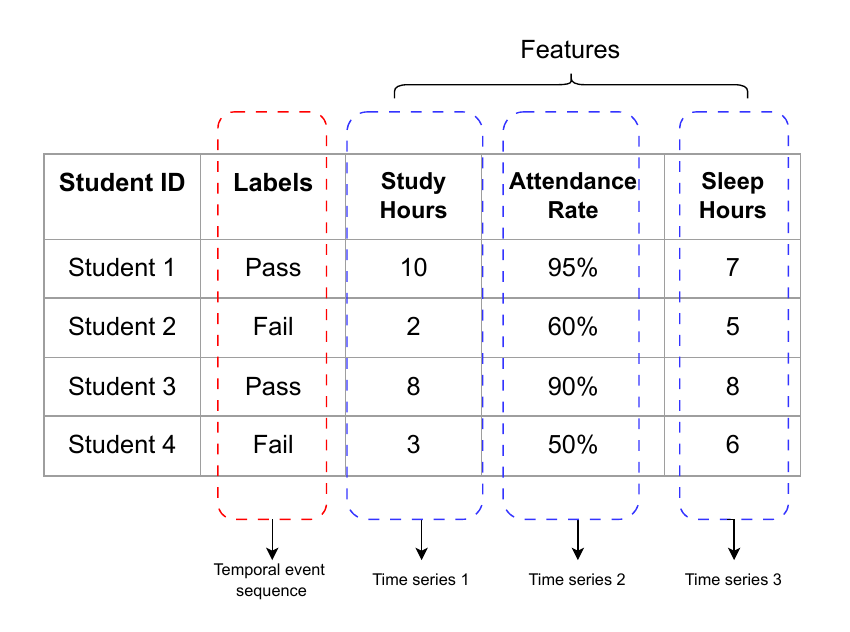}
\captionsetup{font=small}
  \caption{A feature selection example for classifying if a student would fail in the exam. Each feature dimension (blue box) is treated as a time series. The label dimension (red box) is treated as the temporal event sequence. Estimation is between each pair of labels and features (temporal event sequence-time series). The MI can reflect the feature importance.}
\label{fig:feature_selection}
\end{figure}

All the datasets we used are from the widely used benchmark~\cite{li2018feature}, and we selected those with continuous variables as features. The statistics information of the selected datasets is shown in Table~\ref{tab:statistics}.

The detailed experimental results compared with Ross using random forest (RF), support vector machine (SVM), and logistic regression (LR) as the base model are shown in Table~\ref{tab:compare_feature} and~\ref{tab:compare_feature1}.


\begin{table}[H]
\centering
\captionsetup{font=small}
\caption{
Detailed classification accuracy of random forest (RF), support vector machine (SVM), and logistic regression (LR) with top k selected features based on the existing mutual information method and the proposed method.}
\label{tab:compare_feature}
\resizebox{0.9\columnwidth}{!}{
\begin{tabular}{|l|l|c|c|c|c|c|}
\hline
Dataset                    & Method        & k=20   & k=40   & k=60   & k=80   & k=100  \\ \hline
\multirow{12}{*}{COIL20}   & RF Ross:      & 0.9201 & 0.9514 & 0.9688 & 0.9722 & 0.9757 \\ \cline{2-7} 
                           & RF Mixture:   & 0.9306 & 0.9618 & 0.9722 & 0.9792 & 0.9757 \\ \cline{2-7} 
                           & RF NoClu.:    & 0.941  & 0.9757 & 0.9896 & 0.9931 & 0.9896 \\ \cline{2-7} 
                           & RF proposed:  & 0.9965 & 1.0    & 0.9965 & 1.0    & 1.0    \\ \cline{2-7} 
                           & SVM Ross:     & 0.7535 & 0.8576 & 0.9062 & 0.9132 & 0.9271 \\ \cline{2-7} 
                           & SVM Mixture:  & 0.7326 & 0.875  & 0.8958 & 0.9167 & 0.9306 \\ \cline{2-7} 
                           & SVM NoClu:    & 0.8368 & 0.9306 & 0.9549 & 0.9549 & 0.9722 \\ \cline{2-7} 
                           & SVM proposed: & 0.9028 & 0.9514 & 0.9722 & 0.9896 & 0.9896 \\ \cline{2-7} 
                           & LR Ross:      & 0.6528 & 0.7847 & 0.8125 & 0.8472 & 0.8854 \\ \cline{2-7} 
                           & LR Mixture:   & 0.6562 & 0.7847 & 0.8194 & 0.8368 & 0.8889 \\ \cline{2-7} 
                           & LR NoClu:     & 0.7465 & 0.8403 & 0.8889 & 0.8993 & 0.9062 \\ \cline{2-7} 
                           & LR proposed:  & 0.8056 & 0.8889 & 0.934  & 0.9479 & 0.9618 \\ \hline
\multirow{12}{*}{ORL}      & RF Ross:      & 0.6375 & 0.75   & 0.7875 & 0.8125 & 0.875  \\ \cline{2-7} 
                           & RF Mixture:   & 0.575  & 0.7    & 0.775  & 0.8125 & 0.8375 \\ \cline{2-7} 
                           & RF NoClu.:    & 0.75   & 0.8125 & 0.85   & 0.9    & 0.9    \\ \cline{2-7} 
                           & RF proposed:  & 0.7125 & 0.8625 & 0.9375 & 0.9    & 0.9375 \\ \cline{2-7} 
                           & SVM Ross:     & 0.6875 & 0.725  & 0.8125 & 0.9    & 0.925  \\ \cline{2-7} 
                           & SVM Mixture:  & 0.6375 & 0.8125 & 0.875  & 0.8375 & 0.85   \\ \cline{2-7} 
                           & SVM NoClu:    & 0.7    & 0.85   & 0.8375 & 0.9    & 0.9125 \\ \cline{2-7} 
                           & SVM proposed: & 0.7375 & 0.8875 & 0.925  & 0.925  & 0.9375 \\ \cline{2-7} 
                           & LR Ross:      & 0.5125 & 0.6625 & 0.725  & 0.775  & 0.8    \\ \cline{2-7} 
                           & LR Mixture:   & 0.4875 & 0.6875 & 0.75   & 0.7625 & 0.825  \\ \cline{2-7} 
                           & LR NoClu:     & 0.625  & 0.7375 & 0.8    & 0.8125 & 0.8375 \\ \cline{2-7} 
                           & LR proposed:  & 0.675  & 0.8375 & 0.9375 & 0.9375 & 0.95   \\ \hline
\multirow{12}{*}{Yale}     & RF Ross:      & 0.6061 & 0.697  & 0.7273 & 0.7576 & 0.7273 \\ \cline{2-7} 
                           & RF Mixture:   & 0.4242 & 0.6667 & 0.7576 & 0.6667 & 0.697  \\ \cline{2-7} 
                           & RF NoClu.:    & 0.5152 & 0.6667 & 0.6364 & 0.697  & 0.7879 \\ \cline{2-7} 
                           & RF proposed:  & 0.5152 & 0.6364 & 0.7273 & 0.6364 & 0.6364 \\ \cline{2-7} 
                           & SVM Ross:     & 0.4848 & 0.6667 & 0.6364 & 0.6364 & 0.6061 \\ \cline{2-7} 
                           & SVM Mixture:  & 0.4848 & 0.5758 & 0.6061 & 0.5758 & 0.5758 \\ \cline{2-7} 
                           & SVM NoClu:    & 0.4848 & 0.6667 & 0.6364 & 0.5758 & 0.6061 \\ \cline{2-7} 
                           & SVM proposed: & 0.5758 & 0.5455 & 0.6061 & 0.6364 & 0.6364 \\ \cline{2-7} 
                           & LR Ross:      & 0.4848 & 0.6061 & 0.6364 & 0.697  & 0.6364 \\ \cline{2-7} 
                           & LR Mixture:   & 0.5152 & 0.5455 & 0.6364 & 0.6364 & 0.6061 \\ \cline{2-7} 
                           & LR NoClu:     & 0.4242 & 0.697  & 0.6061 & 0.6061 & 0.6364 \\ \cline{2-7} 
                           & LR proposed:  & 0.5152 & 0.6061 & 0.6364 & 0.697  & 0.7576 \\ \hline
\multirow{12}{*}{warpPIE}  & RF Ross:      & 0.8571 & 1.0    & 1.0    & 1.0    & 1.0    \\ \cline{2-7} 
                           & RF Mixture:   & 0.9286 & 1.0    & 1.0    & 1.0    & 1.0    \\ \cline{2-7} 
                           & RF NoClu.:    & 0.9286 & 1.0    & 1.0    & 1.0    & 1.0    \\ \cline{2-7} 
                           & RF proposed:  & 0.9048 & 1.0    & 1.0    & 1.0    & 1.0    \\ \cline{2-7} 
                           & SVM Ross:     & 0.9286 & 0.9524 & 0.9762 & 0.9762 & 0.9762 \\ \cline{2-7} 
                           & SVM Mixture:  & 0.9762 & 1.0    & 1.0    & 1.0    & 1.0    \\ \cline{2-7} 
                           & SVM NoClu:    & 0.9762 & 1.0    & 1.0    & 1.0    & 1.0    \\ \cline{2-7} 
                           & SVM proposed: & 1.0    & 1.0    & 1.0    & 1.0    & 1.0    \\ \cline{2-7} 
                           & LR Ross:      & 0.9524 & 1.0    & 0.9762 & 0.9524 & 0.9762 \\ \cline{2-7} 
                           & LR Mixture:   & 1.0    & 1.0    & 1.0    & 1.0    & 1.0    \\ \cline{2-7} 
                           & LR NoClu:     & 0.9762 & 1.0    & 1.0    & 1.0    & 1.0    \\ \cline{2-7} 
                           & LR proposed:  & 0.9762 & 1.0    & 1.0    & 1.0    & 1.0    \\ \hline
\end{tabular}
}
\end{table}

\begin{table}[h]
\centering
\captionsetup{font=small}
\caption{ Statistics information of datasets for feature selection.}
\label{tab:statistics}
\resizebox{0.8\columnwidth}{!}{%
\begin{tabular}{|l|c|c|c|}
\hline
Dataset           & Num. of  Instances & Num. of  Features & \multicolumn{1}{l|}{Num. of Classes} \\ \hline
COIL20     & 1440               & 1024              & 20                                   \\ \hline
ORL        & 400                & 1024              & 40                                   \\ \hline
Yale       & 165                & 1024              & 15                                   \\ \hline
warpPIE10P & 210                & 2420              & 10                                   \\ \hline
Isolet     & 1560               & 617               & 26                                   \\ \hline
TOX\_171   & 171                & 5748              & 4                                    \\ \hline
USPS       & 9298               & 256               & 10                                   \\ \hline
CLL\_SUB\_111    & 111               & 11340              & 3                                    \\ \hline
\end{tabular}%
}
\end{table}

\begin{table}[H]
\centering
\captionsetup{font=small}
\caption{
Detailed classification accuracy of random forest (RF), support vector machine (SVM), and logistic regression (LR) with top k selected features based on the existing mutual information method and the proposed method.}
\label{tab:compare_feature1}
\resizebox{0.9\columnwidth}{!}{
\begin{tabular}{|l|l|c|c|c|c|c|}
\hline
Dataset                    & Method        & k=20   & k=40   & k=60   & k=80   & k=100  \\ \hline
\multirow{12}{*}{Isolet}   & RF Ross:      & 0.6122 & 0.7404 & 0.7853 & 0.7981 & 0.8109 \\ \cline{2-7} 
                           & RF Mixture:   & 0.6827 & 0.7404 & 0.7917 & 0.8301 & 0.8173 \\ \cline{2-7} 
                           & RF NoClu.:    & 0.4199 & 0.5096 & 0.5897 & 0.7276 & 0.7821 \\ \cline{2-7} 
                           & RF proposed:  & 0.5801 & 0.7436 & 0.8013 & 0.7981 & 0.8237 \\ \cline{2-7} 
                           & SVM Ross:     & 0.6154 & 0.7179 & 0.766  & 0.7917 & 0.8141 \\ \cline{2-7} 
                           & SVM Mixture:  & 0.6538 & 0.7468 & 0.8013 & 0.8333 & 0.8269 \\ \cline{2-7} 
                           & SVM NoClu:    & 0.4615 & 0.4904 & 0.5641 & 0.734  & 0.8269 \\ \cline{2-7} 
                           & SVM proposed: & 0.5705 & 0.7724 & 0.8301 & 0.8526 & 0.8494 \\ \cline{2-7} 
                           & LR Ross:      & 0.5962 & 0.7212 & 0.7724 & 0.8109 & 0.8397 \\ \cline{2-7} 
                           & LR Mixture:   & 0.641  & 0.7628 & 0.7949 & 0.8365 & 0.8558 \\ \cline{2-7} 
                           & LR NoClu:     & 0.4231 & 0.4872 & 0.5321 & 0.7372 & 0.8301 \\ \cline{2-7} 
                           & LR proposed:  & 0.5513 & 0.7532 & 0.8301 & 0.8718 & 0.8718 \\ \hline
\multirow{12}{*}{TOX}      & RF Ross:      & 0.6857 & 0.8286 & 0.8571 & 0.8    & 0.6857 \\ \cline{2-7} 
                           & RF Mixture:   & 0.3714 & 0.4    & 0.4857 & 0.4857 & 0.5714 \\ \cline{2-7} 
                           & RF NoClu.:    & 0.7429 & 0.7429 & 0.7714 & 0.8    & 0.7714 \\ \cline{2-7} 
                           & RF proposed:  & 0.6286 & 0.8857 & 0.8    & 0.7714 & 0.8286 \\ \cline{2-7} 
                           & SVM Ross:     & 0.6    & 0.7143 & 0.7143 & 0.6857 & 0.7143 \\ \cline{2-7} 
                           & SVM Mixture:  & 0.4286 & 0.4    & 0.6    & 0.5429 & 0.6286 \\ \cline{2-7} 
                           & SVM NoClu:    & 0.7714 & 0.7143 & 0.8571 & 0.8    & 0.8    \\ \cline{2-7} 
                           & SVM proposed: & 0.7143 & 0.8571 & 0.8286 & 0.6571 & 0.7429 \\ \cline{2-7} 
                           & LR Ross:      & 0.7143 & 0.7429 & 0.7429 & 0.8286 & 0.8    \\ \cline{2-7} 
                           & LR Mixture:   & 0.4857 & 0.4571 & 0.6571 & 0.6571 & 0.6571 \\ \cline{2-7} 
                           & LR NoClu:     & 0.7429 & 0.7429 & 0.8571 & 0.7714 & 0.8571 \\ \cline{2-7} 
                           & LR proposed:  & 0.6571 & 0.8857 & 0.8    & 0.7429 & 0.8286 \\ \hline
\multirow{12}{*}{USPS}     & RF Ross:      & 0.7129 & 0.7801 & 0.821  & 0.9065 & 0.9199 \\ \cline{2-7} 
                           & RF Mixture:   & 0.7145 & 0.7871 & 0.872  & 0.8989 & 0.9226 \\ \cline{2-7} 
                           & RF NoClu.:    & 0.728  & 0.7995 & 0.8796 & 0.8995 & 0.9075 \\ \cline{2-7} 
                           & RF proposed:  & 0.722  & 0.8882 & 0.9    & 0.922  & 0.9317 \\ \cline{2-7} 
                           & SVM Ross:     & 0.5565 & 0.6419 & 0.743  & 0.8946 & 0.9124 \\ \cline{2-7} 
                           & SVM Mixture:  & 0.5548 & 0.6414 & 0.8199 & 0.871  & 0.9108 \\ \cline{2-7} 
                           & SVM NoClu:    & 0.5715 & 0.6387 & 0.7806 & 0.8188 & 0.8538 \\ \cline{2-7} 
                           & SVM proposed: & 0.6    & 0.8651 & 0.8968 & 0.9151 & 0.9409 \\ \cline{2-7} 
                           & LR Ross:      & 0.5215 & 0.6199 & 0.7145 & 0.8849 & 0.9048 \\ \cline{2-7} 
                           & LR Mixture:   & 0.5339 & 0.6156 & 0.7941 & 0.8554 & 0.9    \\ \cline{2-7} 
                           & LR NoClu:     & 0.5441 & 0.6199 & 0.7527 & 0.7898 & 0.8124 \\ \cline{2-7} 
                           & LR proposed:  & 0.5796 & 0.85   & 0.8823 & 0.9102 & 0.9306 \\ \hline
\multirow{12}{*}{CLL\_SUB} & RF Ross:      & 0.6522 & 0.6522 & 0.6522 & 0.5652 & 0.6087 \\ \cline{2-7} 
                           & RF Mixture:   & 0.4783 & 0.4348 & 0.3478 & 0.3478 & 0.4348 \\ \cline{2-7} 
                           & RF NoClu.:    & 0.6522 & 0.7391 & 0.6957 & 0.6522 & 0.7391 \\ \cline{2-7} 
                           & RF proposed:  & 0.6522 & 0.6957 & 0.6957 & 0.6522 & 0.6522 \\ \cline{2-7} 
                           & SVM Ross:     & 0.5652 & 0.6522 & 0.6087 & 0.6087 & 0.6957 \\ \cline{2-7} 
                           & SVM Mixture:  & 0.4348 & 0.3478 & 0.2609 & 0.3478 & 0.5217 \\ \cline{2-7} 
                           & SVM NoClu:    & 0.6087 & 0.6522 & 0.6087 & 0.6087 & 0.6522 \\ \cline{2-7} 
                           & SVM proposed: & 0.4783 & 0.6957 & 0.7391 & 0.6087 & 0.7391 \\ \cline{2-7} 
                           & LR Ross:      & 0.6522 & 0.5652 & 0.4783 & 0.6087 & 0.6087 \\ \cline{2-7} 
                           & LR Mixture:   & 0.4783 & 0.3913 & 0.1739 & 0.3478 & 0.4783 \\ \cline{2-7} 
                           & LR NoClu:     & 0.6957 & 0.6957 & 0.5217 & 0.5652 & 0.6087 \\ \cline{2-7} 
                           & LR proposed:  & 0.5652 & 0.6957 & 0.6957 & 0.6087 & 0.6522 \\ \hline
\end{tabular}
}
\end{table}



\end{document}